%% file: main_arxiv.tex
\documentclass{article}

\PassOptionsToPackage{numbers, sort&compress}{natbib}

 \usepackage[preprint]{neurips_2025}

\usepackage[utf8]{inputenc} 
\usepackage[T1]{fontenc}    
\usepackage{hyperref}       
\usepackage{url}            
\usepackage{booktabs}       
\usepackage{amsfonts}       
\usepackage{nicefrac}       
\usepackage{microtype}      
\usepackage{xcolor}         

\usepackage[normalem]{ulem}
\usepackage{lmodern}

\usepackage{times}
\usepackage{amsmath, nccmath, amssymb}
\usepackage{amsthm}
\usepackage{geometry}
\usepackage{comment}
\usepackage{enumitem}
\usepackage{color}
\usepackage{graphicx}
\usepackage{bbold}
\usepackage{xfrac}
\usepackage{thmtools}
\usepackage{algorithm}
\usepackage{algpseudocode}
\usepackage{mathtools}
\usepackage{mathrsfs}  
\usepackage{caption}
\usepackage{makecell}
\usepackage{multirow}
\usepackage{thm-restate}

\usepackage{notation}
\usepackage[textsize=tiny]{todonotes}

\newcommand{\shirley}[2][]{\todo[color=orange!20,#1]{{\bf SL:} #2}}

\newtheorem{theorem}{Theorem}
\newtheorem{lemma}{Lemma}

\newcommand{\Lx}{\beta}

\newcommand{\dx}{d}

\usepackage{tikz}

\usetikzlibrary{arrows.meta}
\tikzset{
  >=Stealth,                                   
  figarrow/.style={-{Triangle[length=4.5pt,width=3.8pt]}, line width=0.8pt}
}

\DeclareRobustCommand{\tikzto}[1][1.7em]{%
  \mathrel{\tikz[baseline=-0.6ex]{\draw[figarrow] (0,0) -- (#1,0);}}%
}

\title{Non-stationary Bandit Convex Optimization:\\A Comprehensive Study}

\author{%
Xiaoqi Liu\thanks{University of Oxford. Correspondence to: \texttt{\{shirley.liu, arya.akhavan\}@stats.ox.ac.uk}.}
\And
Dorian Baudry$^*$\thanks{Univ. Grenoble Alpes, Inria, CNRS, Grenoble INP, LIG, 38000 Grenoble, France.}\And
Julian Zimmert\thanks{Google Research.}
\And
Patrick Rebeschini$^*$\And
Arya Akhavan$^*$\thanks{École Polytechnique de Paris, IP Paris.}
}

\begin{document}

\maketitle

\begin{abstract}
Bandit Convex Optimization is a fundamental class of sequential decision-making problems, where the learner selects actions from a continuous domain and observes a loss (but not its gradient) at only one point per round. We study this problem in non-stationary environments, and aim to minimize the regret under three standard measures of non-stationarity: the number of switches $S$ in the comparator sequence, the total variation $\Delta$ of the loss functions, and the path-length $P$ of the comparator sequence. We propose a polynomial-time algorithm, Tilted Exponentially Weighted Average with Sleeping Experts (TEWA-SE), which adapts the sleeping experts framework from online convex optimization to the bandit setting.  
For strongly convex losses, we prove that TEWA-SE is minimax-optimal with respect to known $S$ and $\Delta$ by establishing matching upper and lower bounds. 
By equipping TEWA-SE with the Bandit-over-Bandit framework, we extend our analysis to environments with unknown non-stationarity measures. 
For general convex losses, we  introduce a second algorithm, clipped Exploration by Optimization (cExO), based on exponential weights over a discretized action space. While not polynomial-time computable, this method achieves minimax-optimal  regret with respect to known $S$ and $\Delta$, and improves  on the best existing bounds  with respect to $P$.
\end{abstract}

\input{intro.tex}

\input{upper_bound_TEWA}
\input{lower_bound.tex}

\input{upper_bound_exp_by_opt}
\input{conclusion}





\bibliographystyle{unsrtnat}
\bibliography{bibliography}



\clearpage

\appendix


\input{appx_definitions}
\input{appx_reductions}

\input{appx_details_TEWA}
\input{appx_upper_bound_TEWA}

\input{appx_upper_bound_unknown_S}
\input{appx_lower_bound}

\input{appx_exo}



\end{document}

%% file: intro.tex
\section{Introduction}\label{sec:intro}
Many real-world decision-making problems, such as resource allocation, 
experimental design, 
 or hyperparameter tuning  require repeatedly selecting an action from a continuous space under uncertainty and limited feedback. These settings are naturally modeled as 
Bandit Convex Optimization (see \cite{lattimore24introbco} for an introduction), 
in which an  adversary  fixes a sequence of $T$ loss functions $f_1, f_2, \dots, f_T: \mathbb{R}^d \to \mathbb{R}$ beforehand, and a learner sequentially interacts with the adversary  for $T$ rounds. 
At each round $t$, the learner selects an action $\bz_t$ from a continuous arm set $\com \subseteq \mathbb{R}^d$, assumed to be convex and compact. The learner then incurs a loss  $f_t(\bz_t)$ and observes a  noisy feedback: 
\begin{align}
y_t = f_t(\bz_t) + \xi_t\,, 
\end{align} where $\xi_t$ is a sub-Gaussian noise variable (Definition~\ref{def::subgaussian}). The goal is to minimize the learner's \textit{regret} with respect to (w.r.t.)\@ a performance benchmark. 
In the online learning literature \cite{Hazan16intoduction, orabona2019modern}, the benchmark is typically the best \textit{static} action in hindsight, with cumulative loss $\min_{\bz\in\com}\sum_{t=1}^{T} f_t(\bz)$.

However,  non-stationarity arises in many applications where  different actions may work well during different time intervals. Hence,  a line of works  \cite{zinkevich2003online, Mokhtari16, Jadbabaie15, besbes2015non, hall2015online} propose to compare the learner's actions against   a sequence of comparators $\bu_1, \dots, \bu_T \in \Theta$, leading to a regret defined as
\begin{align}\label{eq:Reg_universal}
   R(T, \bu_{1:T})\coloneqq \sum_{t=1}^{T}\Exp\left[f_t(\bz_t) -  f_t(\bu_t)\right]\,,
\end{align}
where   $\bu_{1:T}$ denotes $(\bu_t)_{t=1}^T$, and 
 the expectation is taken w.r.t.\@ the randomness in the learner’s actions $\bz_t$'s and the randomness of the noise variables $\xi_t$'s, similarly to the standard notion of  pseudo-regret in the bandit literature, see e.g., \cite[Section  4.8]{lattimore2020bandit}. Choosing the regret-maximizing comparator in \eqref{eq:Reg_universal} gives rise to the notion of \textit{dynamic regret}, defined as
\begin{align}\label{eq:DynReg}
   R^\dyn(T)\coloneqq \max_{\bu_{1:T}\in \Theta^T} R(T,  \bu_{1:T})\,.
\end{align}
While addressing non-stationarity through dynamic regret has been extensively studied in multi-armed bandits (e.g., \cite{garivier2008upper, besbes2014stochastic, BessonKMS22}), it remains relatively underexplored in continuum bandits \cite{zhao2021bandit, besbes2015non, wang2025adaptivity}. This work aims to bridge this gap by proposing algorithms for Bandit Convex Optimization that achieve sublinear dynamic regret.
Such a rate is generally unattainable without imposing  structural constraints on the environment, i.e.,  the comparator sequence  and the loss function sequence  \cite{Jadbabaie15}.  
For the  comparator sequence, two commonly studied constraints   are the number of switches \cite{herbster1998tracking} and the path-length \cite{zinkevich2003online}, defined respectively as
\begin{align}\label{eq:def_S}
 S(\bu_{1:T})\coloneqq 1+ \sum_{t=2}^T \mathbf{1}\left(\bu_t \neq \bu_{t-1}\right)  \le S\,,\quad\quad   P(\bu_{1:T})\coloneqq\sum_{t=2}^T\|\bu_t - \bu_{t-1}\|\le P\,.
\end{align}
For the loss function sequence, a popular constraint   is the total variation \cite{besbes2015non}, defined as
\begin{align} 
     \Delta(f_{1:T}) \coloneqq \sum_{t=2}^{T}\max_{\bz\in\com}\left|f_t(\bz) - f_{t-1}(\bz)\right|\le \Delta\,.
\end{align}
The  constraints that the upper bounds  $S$, $P$ and $\Delta$ respectively impose on the comparators or on the loss functions lead to different notions of regret.  
We call the regret  for environments  constrained by $S$ the \textit{switching regret}, which we define as
\begin{align}\label{eq:def_R_S}
R^\swi(T,S)\coloneqq\max_{\bu_{1:T}:S(\bu_{1:T})\leq S} R(T, \bu_{1:T})\,.
\end{align} 
Similarly, we call the regret for environments constrained by $P$ the \textit{path-length regret}, denoted by   $R^\pat(T, P)$. 
We also use $R^\dyn(T, \Delta)$ and $R^{\dyn}(T, \Delta, S)$ to denote the dynamic regret, where the arguments after $T$ specify environment constraints. See \eqref{eq:R_dyn_T_Delta} for rigorous definitions.

We detail in Section~\ref{sec:conversions} conversion results between the different regret definitions that we introduced:  a sublinear  switching regret $R^\swi(T, S)$ implies  sublinear  $R^{\dyn}(T, \Delta), R^{\dyn}(T, \Delta, S)$ and $R^\pat(T, P) $, as  illustrated in Figure~\ref{fig:flowchart} (see also \cite{zhang2018dynamic, zhang2020minimizing}). 
%
Furthermore, the upper bounds on the switching regret presented in this work are derived from upper bounds on the \textit{adaptive regret} \cite{HazanS09, daniely2015strongly}, 
 which is defined using  an interval length   $\sfB\in [T]$  as follows,
    \begin{align}\label{eq:def_ada_regret}
            R^{\ada} (\sfB, T)\coloneqq \max_{\substack{p,q\in[T],\\
            0<q-p\leq \sfB}}\max_{\bu\in \Theta}\sum_{t=p}^{q} \Exp\left[f_t(\bz_t) -   f_t(\bu)\right]\,.
        \end{align}
With an appropriate tuning of $\sfB$ depending on   $S$, an adaptive regret sublinear in $\sfB$ implies a switching regret  sublinear in $T$  through a simple reduction, see e.g., discussions in \cite{daniely2015strongly}. We illustrate the relations between these regret notions in Figure \ref{fig:flowchart}.

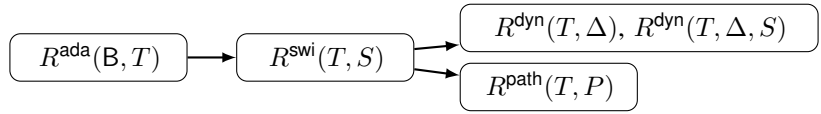
\begin{figure}[htbp]
\centering        
\begin{tikzpicture}[
    node distance=3cm,
    block/.style={rectangle, draw, fill=black!0, text width=6em, text centered, rounded corners, minimum height=1em},
    arrow/.style={->, >=latex, thick}
]
\node[block] (A) {$R^\ada(\sfB, T)$};
\node[block, right of=A] (B) {$R^\swi(T, S)$};
\node[block, above right=-0.25cm and 0.6cm of B, text width=4.5cm] (C) {$R^{\dyn}(T,\Delta)$, $R^\dyn(T, \Delta, S)$};
\node[block, below right=-0.25cm and 0.6cm of B] (D) {$R^\pat(T, P)$};

\draw[arrow] (A) -- (B);
\draw[arrow] (B) -- (C);
\draw[arrow] (B) -- (D);
\end{tikzpicture}
\caption{Conversions between regrets:   $R_1$$\tikzto$$R_2$   means that if regret  $R_1$  is sublinear in $T$ (or $\sfB$), then  regret  $R_2$   is also sublinear in $T$, 
see  Proposition~\ref{prop:conversions} for precise mathematical statements.}
\label{fig:flowchart}
\end{figure}

We conclude this section by detailing the main  notation and assumptions  used throughout the paper.

\paragraph{Notation} For $k\in\mathbb{N}^+$, we denote by $[k]$ the set of  positive integers  $\le k$. 
 We denote  by $\|\cdot\|$ the Euclidean norm, $\mathbb{B}^d=\{\bu\in \bbR^d:  \|\bu\|\le 1\}$ the unit Euclidean ball, and $\Pi_{\Theta}(\bx) =\argmin_{\bw\in \Theta} \|\bw- \bx\| $ the Euclidean projection of $\bx$ to $\Theta$. We use $ a\vee b\equiv \max(a,b)$ and $a\wedge b\equiv \min(a, b)$. 
%
    If $A,B$ depend on $T$, we use  $A=\cO(B)$  (resp.\@   $A=\Omega(B)$) 
    when there exists $c>0$ s.t.\@ $A\leq cB$ 
    (resp.\@ $\geq$) 
    with  $c$  independent of $ T, d, S, \Delta$ and  $P$. To hide polylogarithmic factors in $T$, we use interchangeably  $A=\wt \cO(B)$ and $A\lesssim B$  
    (resp.\@  $A=\wt \Omega(B)$ and $A\gtrsim B$),
    e.g., $A\leq T\log T\Longrightarrow A=\wt \cO(T)$. Moreover, $A=o(B)$ means $A/B\to 0$ as $T\to \infty$.

\paragraph{Assumptions}
For simplicity, we assume that the time horizon $T$ is known in advance; the case of unknown $T$ can be handled using the standard doubling trick \cite{cesaBianchi1997howto}.
 For some $\sigma>0$, the noise variables $(\xi_t)_{t=1}^T$ are $\sigma$-sub-Gaussian. 
For all  $t\in[T]$,   $\max_{\bx\in\com}|f_t(\bx)|\leq 1$. We consider two cases: (i)   general  convex losses $f_t$, where we assume Lipschitz continuity with constant $K$, and (ii)  the special   strongly-convex case, where we assume $\beta$-smoothness.
  The  domain $\Theta$ is assumed to contain a ball of radius $r$ for some constant $r>0$, and has a bounded diameter $\text{diam}(\Theta) \coloneqq \sup\{\|\bx-\bw\| : \bx,\bw \in \Theta\} \le D$  for some constant $D>0$. 

\subsection{Main contributions}\label{sec:contributions}

Existing works on non-stationary  Bandit Convex Optimization   study different aspects of the problem in isolation:   
\cite{besbes2015non, wang2025adaptivity}  focus on dynamic regret $ R^\dyn(T, \Delta)$, while  \cite{zhao2021bandit,chen2018bandit} address  path-length regret $R^\pat(T, P)$. The present work aims to systematically unify and extend previously  scattered results,  establishing a complete picture of the  state-of-the-art regret bounds w.r.t.\@ all three non-stationarity measures $S, \Delta$ and $P$.

Our first contribution  is a polynomial-time  algorithm called  Tilted Exponentially Weighted Average with Sleeping Experts (TEWA-SE), which we design by adapting a series of works  from  online convex optimization 
 \cite{vanerven2016metagrad, wang2020adaptivity, zhang2021dualadaptivity} to the bandit setting with  zeroth-order feedback. 
 It addresses  the absence of gradient information  by employing the randomized perturbation technique  from \cite{flaxman2004online, kleinberg2004nearly} to estimate  gradients, combined with the design of quadratic surrogate loss functions depending on a uniform upper bound on the norm of the gradient estimates.

Following \cite{vanerven2016metagrad, wang2020adaptivity, zhang2021dualadaptivity}, TEWA-SE runs multiple expert algorithms  with different learning rates in parallel, and combines them using  a tilted exponentially weighted average. This  allows TEWA-SE to adapt to the  curvature  of the loss function $f_t$'s without prior knowledge of parameters such as the strong-convexity parameter. For a given interval length $\sfB$, an appropriately tuned TEWA-SE simultaneously 
achieves an adaptive regret of the order $\sqrt{d}\sfB^{\frac{3}{4}}$ for general convex losses and $ d\sqrt{\sfB}$ for strongly-convex losses  (Theorem~\ref{thm:interval_regret}).
Consequently, for a known $S$, we prove that an optimal tuning of TEWA-SE 
yields a switching regret upper bound of  order  $\sqrt{d}S^{\frac{1}{4}}T^{\frac{3}{4}}$  for general convex losses (Corollary~\ref{cor:TEWA_R_swi_known}). In the same result, we further prove that if the losses are strongly convex, and that $\Delta$ is known and incorporated in the tuning of TEWA-SE, the algorithm simultaneously satisfies a $\min\big\{d\sqrt{ST}, d^{\frac{2}{3}}\Delta^{\frac{1}{3}}T^{\frac{2}{3}}\big\}$ dynamic regret bound.
Importantly, TEWA-SE does \textit{not} need to know the actual  strong-convexity parameter, inheriting the adaptivity properties of the framework developed in \cite{vanerven2016metagrad, zhang2021dualadaptivity, wang2020adaptivity}. 
We prove that this dynamic regret upper bound is minimax-optimal in $T, d, S$ and $\Delta$ by establishing a matching lower bound (Theorem~\ref{thm:lowerbound}). Finally, still for strongly-convex losses, we prove that TEWA-SE  can also achieve a  path-length regret of the order $d^{\frac{2}{3}}P^{\frac{1}{3}} T^{\frac{2}{3}}$ when $P$ is known. 
We summarize these results in Table~\ref{tab:summary}. To overcome the restriction of knowing $S,\Delta$ and $P$ to optimally tune TEWA-SE, we also analyze a variant equipped with the Bandit-over-Bandit framework \cite{cheung19bob}. 

\begin{table}[htbp]
    \caption{Regret bounds we obtain for $R^\swi(T, S)$, $R^\dyn(T, \Delta)$ and $ R^\pat(T, P)$, respectively,  for algorithms tuned with known $S,\Delta$ and $P$  (polylogarithmic factors  omitted). Straight underlines indicate  minimax-optimal rates. A wavy underline indicates the result is either new  to the literature (strongly-convex case) or improves on the best-known $P^{\frac{1}{4}}T^{\frac{3}{4}}$ rate \cite{zhao2021bandit} (general convex case). }
    \vspace{5pt}
    \label{tab:summary}
    \centering
    \begin{tabular}{ccc}
        \toprule
          & TEWA-SE (Algorithm \ref{alg:metagrad_lean})&   cExO (Algorithm \ref{alg:convex}) \\
        \midrule
       Convex& $ \sqrt{d}S^{\frac{1}{4}}T^{\frac{3}{4}},\; d^{\frac{2}{5}}\Delta^{\frac{1}{5}} T^{\frac{4}{5}}, \; d^{\frac{2}{5}}P^{\frac{1}{5}} T^{\frac{4}{5}} $ & \rule{0pt}{13pt}\multirow{2}{*}{ $d^\frac{5}{2}\underline{\sqrt{ST}},\; d^{\frac{5}{3}}\underline{\Delta^{\frac{1}{3}}T^{\frac{2}{3}}}, d^{\frac{5}{3}} \uwave{P^{\frac{1}{3}}T^{\frac{2}{3}}}$ }\\
        \rule{0pt}{13pt}Strongly convex & $ \underline{d\sqrt{ST}},\; \underline{d^{\frac{2}{3}}\Delta^{\frac{1}{3}} T^{\frac{2}{3}}},\;  \uwave{d^{\frac{2}{3}}P^{\frac{1}{3}} T^{\frac{2}{3}}}$ &  \\
        \bottomrule
    \end{tabular}
\end{table}


For general convex losses with known $S, \Delta$ and $P$, TEWA-SE achieves a suboptimal $T^{\frac{3}{4}}$ rate (Corollary \ref{cor:TEWA_R_swi_known}), matching the rates in  similar analysis for the static regret \cite{flaxman2004online, kleinberg2004nearly}. Thus, the second contribution of this work is the clipped Exploration by Optimization (cExO) algorithm with improved guarantees for this setting, 
 which  uses  exponential weights   on  a  discretized action space $\Theta$  with clipping \cite{lattimore2021mirror}. For a given interval length $\sfB$, this algorithm with an optimally tuned learning rate  w.r.t.\@ $\sfB$ attains an order $d^{\frac{5}{2}}\sqrt{\sfB}$ adaptive regret (Theorem~\ref{thm:ExO_adaptive_regret}).  When $S, \Delta$ and $P$ are known beforehand, this algorithm with an optimally tuned learning rate achieves the minimax-optimal dynamic regret w.r.t.\@ $ S $ and $ \Delta$ simultaneously, and attains a $P^{\frac{1}{3}}T^{\frac{2}{3}}$ path-length regret  (Corollary~\ref{cor:convex}), improving on the previous best $P^{\frac{1}{4}}T^{\frac{3}{4}}$  \cite{zhao2021bandit}.  While this algorithm is not polynomial-time computable  and has  suboptimal rates w.r.t.\@ the problem dimension $d$,   it provides  insights that may guide future research toward developing efficient algorithms with optimal guarantees for the  convex case.


\subsection{Related work}\label{sec:related_work}

The literature on \textit{Bandit Convex Optimization} (BCO) has  traditionally focused on minimizing the static regret, see the recent monograph \cite{lattimore24introbco} for a comprehensive historical overview. Both convex and strongly convex objective functions have attracted significant attention, beginning with the foundational work of \cite{flaxman2004online} and further developed in subsequent studies such as \cite{agarwal2011stochastic, saha2011improved, fokkema2024online, bubeck2021kernel,  hazan2014bandit, suggala2024second, suggala2021efficient, lattimore2023second}. 
Minimizing regret in non-stationary environments has only received attention more recently \cite{zhao2021bandit,besbes2015non, wang2025adaptivity,chen2018bandit}, see also \cite[Section 2.4]{lattimore24introbco} for an overview for this topic. Among these works,   \cite{besbes2015non, wang2025adaptivity}   study $ R^\dyn(T, \Delta)$, whereas \cite{zhao2021bandit,chen2018bandit} focus on $R^\pat(T, P)$. As we explained above (and formalize in Section~\ref{sec:conversions}), the switching regret $R^\swi(T, S)$  can induce   guarantees on both $R^{\dyn}(T, \Delta)$ and $ R^\pat(T, P)$, but   the reverse does not necessarily hold. Therefore, the results in these works cannot be readily extended to provide regret guarantees  w.r.t.\@ all three  measures $S, \Delta$  and $P$.  

Minimizing regret in  environments with non-stationarity measures such as $S, \Delta$ and $P$  have been addressed with greater depth  in  \textit{Online Convex Optimization} (OCO),  where   the learner has direct access to   gradient information  and can query the gradient or function value at multiple points of the loss function per round. The state-of-the-art algorithm with optimal adaptive regret guarantees is MetaGrad with sleeping experts \cite{zhang2021dualadaptivity}, which queries only one gradient per round, and adapts to  curvature information of the loss function such as strong-convexity when   available. Our polynomial-time algorithm TEWA-SE  builds upon \cite{zhang2021dualadaptivity} and its precursors \cite{vanerven2016metagrad, wang2020adaptivity}, adapting this approach to BCO by  replacing the exact gradient per round with an approximate  gradient estimate, and by  designing a  quadratic surrogate loss. 
The approach in \cite{zhang2021dualadaptivity} follows a long line of successive developments in OCO from expert tracking methods    \cite{herbster1998tracking, bousquet2002tracking, littlestone1994weighted, cesaBianchi1997howto, Vovk1998game, freud1997predictors_specialise, koolen2015squint} to the study of adaptive regret  \cite{HazanS09, adamskiy2016closer,  daniely2015strongly, jun2017online,cutkosky2020parameter, lu2022adaptive, baby2022optimal}, with recent advances \cite{wang2018minimizing, zhao2022efficient, zhang2021dualadaptivity, zhang2018adaptive, yang2024universal} reducing the query complexity from   $\cO(\log T)$ to $\cO(1)$ per round, while achieving optimal adaptive regret or dynamic regret.  
The adaptivity of \cite{zhang2021dualadaptivity} directly inherits from  MetaGrad \cite{vanerven2016metagrad} and its extension  \cite{wang2020adaptivity}, which themselves  build on  earlier adaptive methods \cite{hazan2007adaptive_ogd,do2009proximal}.

For general  convex functions, the approach of substituting a one-point  gradient estimate  for the exact gradient in each round of an OCO algorithm   often yields suboptimal $T^{\frac{3}{4}}$ rates, both in static regret  \cite{flaxman2004online, kleinberg2004nearly} and dynamic regret  \cite{zhao2021bandit, chen2018bandit}; see also our Corollary~\ref{cor:TEWA_R_swi_known}. 
A series of breakthroughs \cite{russo2014learning, bubeck2015bandit,bubeck2018exploratory,bubeck2021kernel, lattimore2021mirror, lattimore2020improved}  indicate that 
$\sqrt{T}$ rates (up to logarithms) are attainable for static regret, at the cost of a higher dependency on $d$.
Our cExO algorithm follows this line of work, using exponential weights  on a discretized action space \cite{lattimore2021mirror}.
By playing inside a clipped domain, we transform the algorithm from one with $\sqrt{T}$ static regret  into one  with $\sqrt{\sfB}$ adaptive regret  (modulo logarithms) for intervals of length $\le\sfB$, which in turn leads to regret guarantees w.r.t.\@ $S, \Delta$ and $P$.

Finally, we  mention that non-stationarity has been widely studied in the \textit{Multi-Armed Bandit} (MAB) literature. A substantial body of work has focused on adapting standard policies—such as UCB \cite{auer2002finite, KL_UCB}, EXP3 \cite{auer2002nonstochastic}, and Thompson Sampling \cite{TS_1933, TS12AG, TS12kaufmann}—to perform effectively under non-stationarity. These adaptations often employ mechanisms to discard outdated information, either \textit{actively} (e.g., change-detection methods \cite{LiuLS18, Cao19, BessonKMS22, auer2019adswitch, WeiL21, SukK22}), or \textit{passively} (e.g., discounted rewards \cite{garivier2008upper, RussacVC19}, sliding windows \cite{TrovoRG20, baudry21ns}, or scheduled restarts \cite{besbes2014stochastic}), but are not straightforward to adapt to BCO.


\subsection{Conversions between different regret definitions}\label{sec:conversions}
We present the key conversions  between different regret notions, illustrated in Figure~\ref{fig:flowchart} above. 
Using the definition of $R^\dyn(T)$ in \eqref{eq:DynReg}, we overload notation slightly to define
\begin{align}\label{eq:R_dyn_T_Delta}
    R^\dyn(T, \Delta)&\coloneqq 
    \sup_{f_{1:T}:\Delta(f_{1:T})\le \Delta} \sum_{t=1}^T\Exp\Big[f_t(\bz_t) - \min_{\bz\in \Theta} f_t(\bz)\Big] \,,
\end{align}
and $R^{\dyn}(T, \Delta, S)$  additionally constrains 
$1 + \sum_{t=2}^T \min_{(\bz_t^*,\bz_{t-1}^*) \in (\mc{Z}_t^*,\mc{Z}_{t-1}^*)} \mathbf{1}(\bz_t^* \neq \bz_{t-1}^*)\le S$ 
where $\mc{Z}_t^*\coloneqq \argmin_{\bz\in \Theta} f_t(\bz)  $ for all $ t\in[T]$. In Proposition~\ref{prop:conversions}, we  show how the adaptive regret  $R^\ada(\sfB, T)$ can be used to bound the switching regret $R^\swi(T, S)$, which in turn can be used to bound the dynamic regret $R^{\dyn}(T, \Delta, S)$ and path-length regret $R^\pat(T, P)$. Consequently,   $R^\ada(\sfB, T)$ and $R^{\swi}(T, S)$ are the primary objects to analyze.

\begin{restatable}{proposition}{propconv}
\label{prop:conversions}
Suppose that an algorithm can be calibrated to satisfy $R^\ada(\sfB, T)\le C \sfB^\kappa$, for any interval length $\sfB \in [T]$, for some factor $C>0$ that is at most polynomial in $d$ and $\log(T)$,  and $\kappa\in [0, 1)$. 

Then, for any $S,S_\Delta, S_P \in [T]$, an appropriate choice of $\sfB$ yields the following regret guarantees:
\begin{align*}
&\textbf{Switching:}\quad \sfB=\ceil{\tfrac{T}{S}}\text{ guarantees that }  
        R^{\swi} (T, S) \leq 2^{1+\kappa}C S^{1-\kappa} T^\kappa\,.\\
&\textbf{Dynamic:}\quad \sfB=\ceil{\tfrac{T}{S}}\vee\ceil{\tfrac{T}{S_\Delta}} \text{ yields } R^{\dyn}(T, \Delta, S) \le R^{\swi} (T, S) \wedge  \left(R^{\swi} (T, S_\Delta) + \Delta\ceil{\tfrac{ T }{S_\Delta}}\right).\\
&\textbf{Path-length:}\quad \sfB=\ceil{\tfrac{T}{S_P}} \text{ ensures that }
    R^{\pat}(T,P) \leq  R^{\swi} (T, S_P)+  \tfrac{P}{r}\cdot \ceil{\tfrac{T}{S_P}}.
\end{align*}

\end{restatable}
The proof is provided  in Appendix~\ref{app::conversion}. We note that the reduction  from $R^\pat(T, P)$ to $R^\swi(T, S)$   in Proposition \ref{prop:conversions} is new and  employs   simple  geometric arguments (see  Lemma \ref{lem: swi to dyn and path simple} in Appendix~\ref{app::conversion}). This reduction  simplifies the analysis of  $R^\pat(T, P)$, though it can yield  slightly looser bounds on $R^\pat(T, P)$ than a  direct analysis, as discussed in \cite{zhang2020minimizing}.




%% file: upper_bound_TEWA.tex
\section{The TEWA-SE  algorithm}\label{sec:AdaptiveReg}
In this section, we develop a polynomial-time algorithm called  Tilted Exponentially Weighted Average  with Sleeping Experts (TEWA-SE, Algorithm~\ref{alg:metagrad_lean}), building on the  two-layer structure of previous experts-based algorithms \cite{HazanS09,  daniely2015strongly, jun2017online}. 
Each expert in TEWA-SE is uniquely defined by its lifetime  and learning rate.   We denote the active experts at time $t$ by $E_1, E_2, \dots, E_{n_t}$, where $E_i$ operates over   interval  $I_i$ with   learning rate $\eta_i$.  In each round $t$, the active  experts  each propose an action, denoted by $\bx_{t, I_i}^{\eta_i}$, and a meta-algorithm  aggregates them into  a single meta-action $\bx_t$ by computing their tilted exponentially weighted average  \cite{vanerven2016metagrad, zhang2021dualadaptivity}, see line~\ref{line:update_xt_lean} in the pseudo-code. Then the algorithm receives  a noisy evaluation of   $f_t$  at $\bx_t$ and constructs an approximate  \textit{gradient estimate} 
$\bg_t\in \bbR^d$ of $f_t$ at $\bx_t$. 
Both  $\bx_t$ and $\bg_t$ are shared with all experts, who update  their actions  via online gradient descent on their \textit{surrogate} loss functions defined using $ \bx_t$ and $\bg_t$. 


TEWA-SE employs the  Geometric Covering scheme from \cite{daniely2015strongly, zhang2021dualadaptivity} to schedule  experts across different time intervals, and the exponential grid from \cite{zhang2021dualadaptivity, vanerven2016metagrad} to assign  varied learning rates to the experts. These deterministic  schemes ensure that  only a \textit{logarithmic} number of  experts are active per round, maintaining computational efficiency.
Intuitively, the meta-algorithm  achieves low adaptive regret on the original loss function because, for each subinterval of times, there exists at least one  individual expert with low static regret on their surrogate loss functions on this subinterval. This is guaranteed by the careful design of the exponential grid of learning rates. 
While full details of TEWA-SE is deferred to Appendix~\ref{sec:details_ada_TEWA}, we highlight below the distinctions between this paper and  prior works.


\begin{algorithm}
\caption{ Tilted Exponentially Weighted Average with Sleeping Experts (TEWA-SE)}
\label{alg:metagrad_lean}
\textbf{Input:} $d,  T, \sfB$,  $h=\min\big(\sqrt{d}\sfB^{-\frac{1}{4}}, r\big)$, 
$\tilde{\Theta}=\{\bu\in \Theta:\bu+h\mathbb{B}^d\subset \Theta\}$, $G$ as in  \eqref{eq:def_hat_G}, expert algorithm $E(I, \eta)$ defined in Algorithm~\ref{alg:ogd_lean}, and $(n_t)_{t\in [T]}$ and  $(I_i, \eta_i)_{i\in [n_t]}\, \forall t\in [T]$ 
\vspace{0.1cm} 
\begin{algorithmic}[1]
\For{$t = 1, 2, \dots,T$}
\For{$E_i\equiv E_i(I_i, \eta_i)\in \{E_1, E_2, \dots, E_{n_t}\}$} \Comment{ $n_t$ experts active at $t$}
\State Receive  action $\bx_{t, I_i}^{\eta_i}$ from expert $E_i$
\If{$\min\{\tau: \tau\in I_i\}=t$}
initialize $L_{t-1, I_i}^{\eta_i}=0 $,  clipped domain $\tilde{\Theta}$ and parameter $G$ \label{line:init_tewa}
\EndIf
\EndFor
\State Set meta-action as $\bx_t = \sum_{i=1}^{n_t} \eta_i\exp(-L_{t-1, I_i}^{\eta_i}) \bx_{t, I_i}^{\eta_i} /\sum_{j=1}^{n_t}\eta_j \exp(-L_{t-1, I_j}^{\eta_j}) $\label{line:update_xt_lean}
\State Sample $\bzeta_t$ uniformly from $\partial \mathbb{B}^d$ 
\State   Query  point $\bz_t=\bx_t +h \bzeta_t $ to obtain $y_t = f_t(\bz_t)+\xi_t$
\State Construct gradient estimate $\bg_t = (d/h) y_t \bzeta_t$

\For{$i=1, 2, \dots, n_t$}
\State Send   meta-action  $\bx_t$ and $\bg_t$ to $E_i$
\State Increment cumulative  loss $ L_{t,I_i}^{\eta_i} = L_{t-1,I_i}^{\eta_i}+\ell_{t}^{\eta_i}(\bx_{t, I_i}^{\eta_i})$ \Comment{$\ell_t^{\eta_i}(\cdot)$ depends on $\bx_t$ and $\bg_t$}
\EndFor
\EndFor
\end{algorithmic}
\end{algorithm}

\paragraph{Construction of one-point gradient estimate}
For a fixed  parameter $h\in (0, r)$, we define the clipped domain  $\tilde{\Theta}=\{\bu\in \Theta: \bu + h\mathbb{B}^d\subset \Theta\}$, where $h< r$  ensures $\tilde{\Theta}\neq \emptyset$. 
In each round $t$, we select a meta-action $\bx_t\in \tilde{\Theta}$ and query the function at a perturbed point  $\bx_t+h\bzeta_t$, receiving noisy  feedback $y_t=f_t(\bx_t+h\bzeta_t)+\xi_t$, where $ \bzeta_t\in \bbR^d $ is sampled uniformly from the unit sphere $\partial \mathbb{B}^d$.  This allows us    to construct the gradient estimate
$\bg_t= (d/h)y_t\bzeta_t$. 
 As implied by \cite[Lemma 1]{flaxman2004online}, the vector $\bg_t$ is an unbiased gradient estimate  of a spherically smoothed version of $f_t$ at $\bx_t$, satisfying 
\begin{align}
    \Exp[\bg_t| \bx_t]=\nabla \hat{f}_t(\bx_t)\,, \quad \text{where}\quad \hat{f}_t(\bx)=\Exp\left[f_t(\bx+h \tilde{\bzeta})\right] \;\forall\, \bx\in \tilde{\Theta}\,, 
\end{align}
with  $\tilde{\bzeta}$  distributed uniformly on the unit ball $\mathbb{B}^d$. Importantly, $\hat{f}_t$ inherits the convexity properties of $f_t$  \cite[Lemmas A.2--A.3]{akhavan2020exploiting}.
Our approach differs from related works in OCO \cite{vanerven2016metagrad, wang2020adaptivity, zhang2021dualadaptivity, wang2018minimizing, zhao2022efficient} that use exact gradients in two key ways: i) in each round, we query the perturbed point $\bz_t=\bx_t+h\bzeta_t$ rather than $\bx_t$, accumulating regret at the perturbed point, and ii) we  constrain $\bx_t$  inside the clipped domain $\tilde{\Theta}$ to ensure all perturbed  $\bz_t$ remain feasible.
In our setting, under the high probability event $\Lambda_T = \big\lbrace|\xi_t| \le 2\sigma \sqrt{\log (T+1)}, \; \forall t\in [T]\big\rbrace$, 
  we have 
    \begin{align}
\|\bg_t\|
      &=(d/h)|f_t(\bx_t + h \bzeta_t) + \xi_t|\le (d/h) \big(1+ 2\sigma \sqrt{\log (T+1)}\big)=: G,\quad \forall t\in[T].\label{eq:def_hat_G}
    \end{align}
This implies a fundamental  tradeoff in selecting the smoothing (and clipping) parameter $h$: larger values reduce $G$ (and  the variance of $\bg_t$), but increase both the approximation error  between $\hat{f}_t$ and $f_t$ and the  error due to clipping, while smaller values reduce  bias at the cost of a higher variance in $\bg_t$. In Theorem~\ref{thm:interval_regret} and Corollary~\ref{cor:TEWA_R_swi_known}, we establish the optimal  $h$  and the resulting   regret guarantees.

    
\begin{algorithm}
\caption{Expert algorithm $  E(I, \eta)$: projected online gradient descent (OGD)} \label{alg:ogd_lean}
\textbf{Input:} 
$I=[r,s]$,   $\eta$,  $G$, clipped domain $\tilde{\Theta}$, and  surrogate loss  $ \ell_t^\eta (\cdot )$ defined in \eqref{eq:surrogate_loss} $\forall  t\in \bbN^+$
\\
\textbf{Initialize:} $\bx_{r, I}^\eta$ be any point in $\tilde{\Theta}$
\begin{algorithmic}[1]
\For{$t=r, r+1, \dots, s$}
\State Send action $\bx_{t, I}^\eta $ to Algorithm~\ref{alg:metagrad_lean}
\State Receive meta-action $\bx_t$ and   $\bg_t$  from Algorithm~\ref{alg:metagrad_lean}
\State \label{line:base_update_lean} Update  $\bx_{t+1, I}^\eta=\Pi_{\tilde{\Theta}} \big(\bx_{t, I}^\eta -\mu_t \nabla \ell_t^\eta (\bx_{t, I}^\eta)\big)$,  where $ \mu_t=1/(2\eta^2G^2 (t-r+1))$
\EndFor
    \end{algorithmic}
\end{algorithm}

\paragraph{Design of expert algorithms and  surrogate losses}

We choose  projected online gradient descent (OGD) as the expert algorithms (Algorithm~\ref{alg:ogd_lean}), i.e., each expert $ E(I, \eta)$ runs OGD  during its lifetime $I$. 
In  the full-information setting, 
where   experts  observe $f_t$ and   gradients are evaluated at all of their actions, each expert  could simply run OGD on the true loss functions. In contrast, for the bandit setting, with  only one  gradient estimate $\bg_t$ of  the smoothed loss $\hat{f}_t$  per round, we need to construct  surrogate losses for the experts.  
The simplest option is the linear surrogate loss 
$\ell_t(\bx) = -\bg_t^\top(\bx_t-\bx)$, but this fails to leverage curvature information and leads to a large $\wt \cO(\sqrt{|I|}) $ static regret for each expert, ultimately yielding  linear adaptive regret. 


To address these limitations,  inspired by \cite{wang2020adaptivity, vanerven2016metagrad, zhang2021dualadaptivity},  we design the following strongly-convex surrogate loss 
$\ell_t^\eta:\bbR^d\to \bbR$: 
\begin{equation}
    \ell_t^\eta (\bx)= -\eta \bg_t^\top (\bx_t-\bx) + \eta^2 G^2 \|\bx_t-\bx\|^2\,,\quad\forall \bx\in \bbR^d,\label{eq:surrogate_loss}
\end{equation}
where  $G$ is the upper bound \eqref{eq:def_hat_G} on $\|\bg_t\|$, and $\eta$ is the learning rate of the expert. We highlight that our choice of the quadratic term in \eqref{eq:surrogate_loss}  differs from the $\eta^2\|\bg_t\|^2\|\bx_t-\bx\|^2$ and $\eta^2(\bg_t^\top(\bx_t-\bx))^2$ in \cite{zhang2021dualadaptivity} and \cite{vanerven2016metagrad}. The latter necessitates  an additional  condition relating $\Exp[\|\bg_t\|]$ and $\Exp[\|\bg_t\|^2]$ (or $\Exp[\bg_t \bg_t^\top]$) to be satisfied in the analysis, see e.g., \cite[Theorem 2]{vanerven2016metagrad}, and may yield suboptimal rates in dimension $d$ for strongly-convex losses, similar to \cite{vanerven2016metagrad}. 
Our choice of the quadratic term, similar to  \cite{wang2020adaptivity},  eliminates  these limitations and  simplifies the proof. 

For a  comparator $\bu\in \Theta$, \eqref{eq:surrogate_loss} implies that   the linearized  regret  associated with $\hat{f}_t$ on interval $I$  can be bounded as:
\begin{align}\label{eq:surrogate_tradeoff}
\sum_{t\in I}\left\langle \Exp[\bg_t| \bx_t, \Lambda_T], \bx_t-\bu\right\rangle\le\tfrac{1}{\eta} \underbrace{\sum_{t\in I}  \Exp\big[ \ell_t^\eta(\bx_t)-\ell_t^\eta (\bu)\mid \bx_t, \Lambda_T\big]}_{\coloneqq\textsf{A}} + \eta G^2 \sum_{t\in I}  \|\bx_t-\bu\|^2\,.
\end{align} 
 Due to the strong-convexity of $\ell_t^\eta$, each expert attains 
 only an $\cO(\log |I|)$  static regret under OGD  with an optimally tuned step size $\mu_t$ (see line~\ref{line:base_update_lean} of Algorithm~\ref{alg:ogd_lean}, and Lemma~\ref{lem:base_regret} in  Appendix \ref{sec:regret_GC_interval}). This ensures term \textsf{A} above 
 is  also of $\cO(\log |I|)$.  
By the convexity  of $\hat{f}_t$ we have 
\begin{align}
\sum_{t\in I}\Exp\big[\hat{f}_t(\bx_t)- \hat{f}_t(\bu)\mid \Lambda_T\big]
 &\le 
 \Exp\Big[\tfrac{1}{\eta}\textsf{A}+ (\eta G^2-\tfrac{\alpha}{2}) \sum_{t\in I}  \|\bx_t-\bu\|^2 \,\big|\, \Lambda_T\Big]\,,\label{eq:hat_ft_upper_bound}
\end{align} 
where $\alpha=0$ for  general convex $\hat{f}_t$ (and $f_t$), and $\alpha>0$ for  strongly-convex. 
Since both $\alpha$ and $\sum_{t\in I}\|\bx_t-\bu\|^2$ are  unknown a priori, 
we use a deterministic exponential grid of $\eta$ values \cite{daniely2015strongly, zhang2021dualadaptivity}, ensuring  at least one expert covering $I$ effectively minimize the RHS of \eqref{eq:hat_ft_upper_bound},  ultimately yielding a sublinear adaptive regret w.r.t.\@ $f_t$. 
We present this result in the following theorem.

\begin{restatable}{theorem}{thmTEWAada}\label{thm:interval_regret}
For any $T\in \bbN^+$ and $\sfB\in [T]$, 
Algorithm~\ref{alg:metagrad_lean} with $h=\min(\sqrt{d}\sfB^{-\frac{1}{4}}, r)$ 
satisfies 
\begin{align}\label{eq:thm_adaptive_regret}
R^\ada(\sfB,T)
\lesssim \sqrt{d}\sfB^{\frac{3}{4}} +   d^2 ,
\end{align}
and if  $f_t$ is $\alpha$-strongly-convex with $\argmin_{\bx\in \bbR^d} f_t(\bx)\in \Theta$ for all $t\in [T]$,\footnote{
      The assumption that  loss minimizers lie inside  $\Theta$  is common in  zeroth-order optimization, see e.g., \cite{shamir2013complexity, besbes2015non,  ito2020optimal}. Without it, our  upper bound analysis 
 would have an extra term depending on  the gradients at the  minimizers. 
 } it furthermore holds that
\begin{align}\label{eq:thm_adaptive_regret_sc}
    R^{\ada}(\sfB, T)\lesssim  \tfrac{d}{\alpha}\sqrt{\sfB} +  \tfrac{1}{\alpha} d^2\,,
\end{align}
where $\lesssim$ conceals  polylogarithmic terms in $\sfB$ and $T$, independent of $d$ and $\alpha$.
\end{restatable}
The proof of Theorem \ref{thm:interval_regret}  can be found in Appendix~\ref{sec:appx_TEWA_ada}. We emphasize that TEWA-SE does not require knowledge of  the strong-convexity parameter $\alpha$. This parameter is only used in the analysis and appear in the upper bound \eqref{eq:thm_adaptive_regret_sc}.
Compared to the $\cO(\sqrt{\sfB \log T})$ and $ \cO(\tfrac{1}{\alpha}\log T\log \sfB)$ adaptive regrets in \cite{zhang2021dualadaptivity} for   general convex and strongly-convex losses respectively, our bounds in Theorem~\ref{thm:interval_regret} reflect the separation between  online first-order and zeroth-order optimization. This mirrors the established gap in static regret analyses,  see e.g.\@ \cite{hazan2007logarithmic} vs.\@ \cite{shamir2013complexity}.
%
      We further note that our bound for the strongly-convex case has a $\tfrac{1}{\alpha}$  dependency,  which is suboptimal compared to the $ \frac{1}{\sqrt{\alpha}}$  dependency in \cite{ito2020optimal, hazan2014bandit} for static regret in BCO for $\alpha \lesssim 1$.

      Applying Proposition~\ref{prop:conversions},  the adaptive regret bounds in Theorem~\ref{thm:interval_regret} lead to   the following bounds for $R^\swi(T, S)$, $R^\dyn(T, \Delta, S)$ and $R^\pat(T, P)$.  In Corollary~\ref{cor:TEWA_R_swi_known}, for clarity we drop the $\ceil{\cdot}$ operators from the expressions for $\sfB$ and assume without loss of generality  $\sfB$ is an integer (proof  in Appendix~\ref{sec:appx_TEWA_cor}).

\begin{restatable}{corollary}{corTEWAswi}\label{cor:TEWA_R_swi_known} Consider any horizon $T\in \bbN^+$ and assume that, for all $t \in [T]$, the loss $f_t$ is convex, or strongly-convex with $\argmin_{\bx\in\bbR^d}f_t(\bx)\in \Theta$. We refer to the second scenario as the strongly-convex (SC) case. Then, Algorithm~\ref{alg:metagrad_lean} tuned with parameter 
$\sfB$ satisfies the following regret guarantees:

\begin{align*}
&\text{\textbf{Switching.}} \quad \sfB=\tfrac{T}{S} \Longrightarrow R^{\swi} (T,S) \lesssim \begin{cases}  \sqrt{d}S^{\frac{1}{4}}T^{\frac{3}{4}} +    d^2 S  \\
    d\sqrt{ST}+ d^2 S \quad \text{(SC)}  \\
\end{cases} \\ 
&\text{\textbf{Dynamic.}}  \begin{cases}
 \sfB=\tfrac{T}{S}\vee \big(\frac{\sqrt{d}T}{\Delta}\big)^{\frac{4}{5}} \Rightarrow        R^\dyn(T, \Delta, S)\lesssim R^\swi(T, S) \wedge    (d^{\frac{2}{5}} \Delta^{\frac{1}{5}} T^{\frac{4}{5}} + d^{\frac{8}{5}}\Delta^{\frac{4}{5}}T^{\frac{1}{5}})  \quad 
 \\
  \sfB= \frac{T}{S} \vee \big(\frac{dT}{\Delta}\big)^{\frac{2}{3}} \Rightarrow    R^\dyn(T, \Delta, S)\lesssim R^\swi(T, S) \wedge  (d^{\frac{2}{3}} \Delta^{\frac{1}{3}} T^{\frac{2}{3}}+ d^{\frac{4}{3}} \Delta^{\frac{2}{3}} T^{\frac{1}{3}}) \quad \text{(SC)}  
\end{cases}  \\
&\text{\textbf{Path-length.}} \;  \begin{cases}
 \sfB=\big(\frac{r\sqrt{d}T}{P}\big)^{\frac{4}{5}} \Rightarrow  R^\pat(T, P)\lesssim 
 r^{-\frac{1}{5}} d^{\frac{2}{5}}P^{\frac{1}{5}}T^{\frac{4}{5}}  
    + r^{-\frac{4}{5}}d^{\frac{8}{5}} P^{\frac{4}{5}}T^{\frac{1}{5}}
    \\   
  \sfB=\big(\frac{rdT}{P}\big)^{\frac{2}{3}}  \Rightarrow   
  R^\pat(T, P)\lesssim r^{-\frac{1}{3}} d^{\frac{2}{3}} P^{\frac{1}{3}}T^{\frac{2}{3}}+  r^{-\frac{2}{3}}d^{\frac{4}{3}} P^{\frac{2}{3}}T^{\frac{1}{3}} \quad \text{(SC)}\,.  \\
\end{cases} 
\end{align*}
\end{restatable}

%% file: lower_bound.tex
 \subsection{Lower bound for strongly-convex loss functions}\label{sec:lower_bound}
In this section, we derive a minimax lower bound on the dynamic regret  and path-length regret, and  discuss the  optimality of  TEWA-SE. 
To derive the lower bound for the dynamic regret, we adopt a standard minimax approach by constructing a class of hard functions, following \cite[Theorem 6.1]{akhavan2020exploiting}. We assume that the adversary  either (i) partitions the time horizon into $S$ segments and assigns a different function from this class to each segment, or (ii) selects a sequence of functions with total variation bounded by $\Delta$.

\begin{restatable}{theorem}{theoremLowerBound}\label{thm:lowerbound}
Let $\com = \mathbb{B}^d$. For $\alpha>0$ denote by $\mathcal{F}_\alpha$ the class of $\alpha$-strongly convex and smooth functions. Let $\pi=\{\bz_t\}_{t=1}^{T}$ be any randomized algorithm (see Appendix~\ref{sec:lowerbounds_proof} for a definition). Then there exists $T_0>0$ such that for all $T\geq T_0$ it holds that 
\begin{align}\label{eq:lwBoundXX}
\sup_{f_1,\dots,f_T\in\mathcal{F}_\alpha} R^{\dyn}(T, \Delta, S) \geq c_1 \cdot \left(\dx \sqrt{ST}\wedge d^{\frac{2}{3}} \Delta^{\frac{1}{3}}T^{\frac{2}{3}}\right)\,,
\end{align}
where $c_1>0$ is a constant independent of $d, T$, $S$ and $\Delta$. 
\end{restatable}
We detail the proof in Appendix~\ref{sec:lowerbounds_proof}. This lower bound establishes that TEWA-SE  achieves the minimax-optimal dynamic regret (up to logarithms) for strongly convex and smooth functions w.r.t.\@ $d$, $T$, $S$ and $\Delta$. We  note that \cite{besbes2015non} derives a lower bound only in terms of 
$T$ and 
$\Delta$, matching \eqref{eq:lwBoundXX}, but it does not explicitly capture the dependence on $d$ nor does it address the interplay between $S$ and $\Delta$. 
In the special case where $S=1$,  Theorem~\ref{thm:lowerbound} recovers the classical minimax static regret  of order $d\sqrt{T}$  \cite{shamir2013complexity,akhavan2020exploiting}. Interestingly, for $d=1$ the scaling of the lower bound as function of $T, S$ and $\Delta$ is  the same as standard lower bounds in the  non-stationary MAB literature \cite{garivier2008upper, besbes2014stochastic}. 
The proof of Theorem~\ref{thm:lowerbound} can be readily adapted to consider only the measure  $S$ with the switching regret, yielding a rate of $d\sqrt{ST}$ and thereby establishing the minimax optimality of TEWA-SE’s switching regret bound.

We also derive a  lower bound for  path-length regret analogously to that for dynamic regret.
In Theorem~\ref{thm:2lowerbound} in Appendix~\ref{sec:lowerbounds_proof} we show that under the same assumptions as in the statement of Theorem~\ref{thm:lowerbound},  
\begin{align}\label{eq:2lowerbound}
   \sup_{f_1,\dots,f_T\in\mathcal{F}_\alpha} R^{\pat}(T,P) \;\ge\; c_2 \cdot   (d^2P)^{\frac{2}{5}}T^{\frac{3}{5}}\,,
\end{align}
where $c_2>0$ is a constant independent of $d$, $T$ and $P$. 
Hence, TEWA-SE may not achieve the optimal regret rate for path-length. Additionally, Eq.~\eqref{eq:2lowerbound} improves upon the only existing $d\sqrt{PT}$ lower bound from \cite{zhao2021bandit} in terms of the horizon $T$, by leveraging a different construction of a hard instance. This improvement comes from assuming $P=o(T)$, which is necessary for sublinear regret. 



\subsection{Parameter-free guarantees}

In Corollary~\ref{cor:TEWA_R_swi_known}, we showed that the knowledge of the non-stationarity measures $S, \Delta$ and $P$ allows optimal tuning of TEWA-SE's parameter $\sfB$. However,  these measures can be hard to estimate. To obtain guarantees without such knowledge, we further analyze TEWA-SE under the 
 \textit{Bandit-over-Bandit} (BoB) framework from \cite{cheung19bob} (see Appendix \ref{appx:TEWA_BoB} for details), 
  which divides  the time horizon into epochs of suitable length $L$ 
  and uses an adversarial bandit algorithm (e.g., EXP3) to select $\sfB$ for TEWA-SE in each epoch from the set  $\mc{B}=\{2^i: i=0, 1, \dots, \lfloor\log_2T\rfloor\}$. 
In Corollary~\ref{cor:BoB} in Appendix~\ref{appx:TEWA_BoB}, we adapt all the upper  bounds in Corollary~\ref{cor:TEWA_R_swi_known} to this framework, and show that this procedure  costs an additional $d^{\frac{1}{3}}T^{\frac{5}{6}}$ term for the general convex case and  $d^{\frac{1}{2}}T^{\frac{3}{4}}$   for the strongly-convex case. Our parameter-free path-length regret  bound $P^{\frac{1}{5}}T^{\frac{4}{5}}+ T^{\frac{5}{6}}$ for the general convex case improves on the $P^{\frac{1}{2}}T^{\frac{3}{4}}$ bound  in \cite{zhao2021bandit} when $P=\Omega(T^{\frac{1}{6}})$.

Recent works on MAB \cite{luo2018efficient, chen2019new, auer2019adswitch, WeiL21, SukK22} have proposed algorithms that achieve optimal dynamic regret  without prior knowledge of  $S$ and $\Delta$. However, they use procedures that crucially rely on the finiteness of the arm set, and are thus ill-suited for BCO. It remains open to determine if the minimax regret rate can be attained without such knowledge in the settings considered in this paper. 

%% file: upper_bound_exp_by_opt.tex
\section{Clipped Exploration by Optimization}\label{sec:convex}
In this section, we propose a second algorithm (Algorithm \ref{alg:convex}) to improve upon the suboptimal rates for $R^\dyn(T, \Delta, S)$ and $R^\pat(T, P)$ that TEWA-SE achieves for   general convex loss functions. For ease of presentation, we assume   in this section that the problem is noiseless,
i.e., $\xi_t=0$ for $t\in[T]$.
We call this algorithm clipped Exploration by Optimization (cExO), which is built on Algorithm 8.3 (ExO)  in \cite{lattimore24introbco}.
The high level idea of ExO is to run exponential weights over a finite discretization of the feasible set, denoted by $\mc{C}\subset \Theta$. 
 We assume the discretization $\mc{C}$
admits a worst-case error of $\varepsilon \coloneqq \sup_{f\in\mc{F}_0} \min_{\bq\in\Delta(\mc{C})}\Exp_{\bz'\sim \bq}f(\bz')-\min_{\bz\in\Theta}f(\bz)$, where  $\mc{F}_0$ denotes  the class of convex and Lipschitz functions, and  $ \Delta(\mc{C})$ denotes the $(|\mc{C}|-1)$-dimensional simplex. 


With  $\bq_0$  initialized as the uniform distribution, in each round $t$, given a loss estimate $\widehat{\bs}_t\in \bbR^{|\mc{C}|}$, ExO  (in its mirror descent formulation) computes $\bq_t=\argmin_{\bq\in\Delta(\mc{C})} \ip{\bq,\widehat{\bs}_{t-1}}+\frac{1}{\eta}\text{KL}(\bq||\bq_{t-1})$, where $\text{KL}$ is the Kullback-Leibler divergence $\text{KL}(\bq||\bp)=\sum_{i=1}^{|\mc{C}|} q_i\log( q_i/ p_i)$ for $ \bq,\bp\in \Delta(\mc{C})$. 
The update rule in cExO departs from the vanilla ExO in this single step, by taking the minimum over the \textit{clipped} simplex $\tilde{\Delta}=\Delta(\mc{C}) \cap [\gamma, 1]^{|\mc{C}|}$ where $\gamma\in (0, \frac{1}{|\mc{C}|})$ is a constant to be tuned, see line \ref{line:ExO_qt} of Algorithm \ref{alg:convex}. Clipping is a standard technique  in mirror descent to ensure the algorithm does not commit too hard to any single action, and therefore detect changes in the environments more easily, yielding  regret guarantees w.r.t.\@ non-stationary measures  \cite[Chapter 31.1]{lattimore2020bandit}.

Given the reference distribution $\bq_t$, cExO selects a playing distribution $\bp_t\in \Delta(\mc{C})$ and an estimator function $E_t\in\mc{E}$ which   returns an updated  loss estimate for each  action in $\mc{C}$, where   $\mc{E}$ denotes the set of functions that map $\mc{C}\times[-1,1]$ to $\mathbb{R}^{|\mc{C}|}$. It does so by solving an intractable optimization problem:\footnote{For detailed discussions about these functions, we refer the reader to \cite{lattimore2021mirror}.} 
\begin{align}\label{eq:def_Lambda_eta}
\argmin_{\bp\in\Delta(\mc{C}),E\in\mc{E}}\Lambda_\eta(\bq_t,\bp,E)\,,
\end{align} where, with $S_{\bq}(\eta \hat \bs) = \max_{\bq'\in\Delta(\mc{C})}\ip{\bq-\bq',\eta\hat \bs} - \text{KL}(\bq'||\bq)$, the objective function is defined by
\begin{align*}
&\Lambda_{\eta}(\bq,\bp,E) \coloneqq \sup_{\bp^\star\in\Delta(\mc{C})}\sup_{f\in\mc{F}_0}\Exp_{\bz\sim \bp}\Big[\ip{\bp-\bp^\star, f}+\ip{\bp^\star-\bq, E(\bz,f(\bz))}+\tfrac{1}{\eta}S_{\bq}(\eta E(\bz,f(\bz)))\Big].
\end{align*}
This optimization problem is intractable due to the large size of $\mc{E}$ and $\mc{F}_0$.\footnote{One can in theory bound the domain of $\mathcal{E}$ and discretize  $\mc{E}$, $\mc{F}_0$ and $\Delta(\mc{C})$. The optimization problem is hence computable, though not in polynomial time.}
The role of this optimization problem is to tradeoff the worst-case cost of 
deviating from the desired distribution $\bq_t$ versus the gain of improved exploration (hence the name Exploration by Optimization).
Finally, cExO samples an action $\bz_t$ according to $\bp_t$, observes the feedback $f(\bz_t)$ and constructs a loss estimate $\widehat{\bs}_t= E_t(\bz_t, f(\bz_t))$ to be used in the subsequent round. cExO achieves  the adaptive regret guarantee stated in Theorem \ref{thm:ExO_adaptive_regret} below.


\begin{algorithm}
\caption{clipped Exploration by Optimization (cExO)} \label{alg:convex}
\textbf{Input:} $d, T,\sfB$, feasible set $\Theta$, a finite covering set $\mc{C}\subset \Theta$ of $\Theta$, discretization error $\varepsilon$, learning rate $\eta$, clipping parameter $\gamma\in (0, \tfrac{1}{|\mc{C}|})$, and $ \tilde{\Delta}=\Delta(\mc{C})\cap [\gamma, 1]^{|\mc{C}|}$\\
\textbf{Initialize:} $q_{0,i}=\frac{1}{|\mc{C}|}\;\forall i\in[|\mc{C}|]\, $.
\begin{algorithmic}[1]
\For{$t=1,\dots, T$}
\State \label{line:ExO_qt}Compute $\bq_t = \argmin_{\bq\in\tilde\Delta}\langle \bq, \widehat{\bs}_{t-1}\rangle + \frac{1}{\eta}\text{KL}(\bq||\bq_{t-1})$
\State\label{line:ExO_pt} Find distribution $\bp_t\in \Delta(\mc{C})$ and $E_t\in \mc{E}$ s.t. $\Lambda_\eta(\bq_t, \bp_t, E_t)\le \inf_{\substack{\bp\in \Delta(\mc{C}),\\ E\in \mc{E}}} \Lambda_\eta(\bq_t, \bp, E) + \eta d$
\State Sample $\bz_t\sim \bp_t$ and observe $ f_t(\bz_t)$
\State Compute $\widehat{\bs}_t = E_t(\bz_t, f_t(\bz_t))$ 
\EndFor
    \end{algorithmic}
\end{algorithm}
\begin{restatable}{theorem}{ExOtheorem}\label{thm:ExO_adaptive_regret}
    For $T\in \bbN^+$ and $\sfB\in[T]$,  Algorithm~\ref{alg:convex} calibrated with     $\varepsilon=\frac{1}{T}$, $\gamma = \frac{1}{T|\mc{C}|}$,  $\eta = \sqrt{\log(\gamma^{-1})/(d^4\log(dT) \sfB)}$ and $\log|\mc{C}|=\cO (d\log(dT^2))$ satisfies
    \begin{align} 
    R^\ada(\sfB, T)\lesssim d^\frac{5}{2}\sqrt{\sfB}\,.
    \end{align}
\end{restatable} 

We then use Proposition~\ref{prop:conversions} to convert the bound of Theorem~\ref{thm:ExO_adaptive_regret} into  the  following regret guarantees w.r.t.\@ $S, \Delta$ and $P$. Like in Corollary~\ref{cor:TEWA_R_swi_known}, we omit $\ceil{\cdot}$ from the expressions for $\sfB$ for clarity. 

\begin{restatable}{corollary}{ExOcorollaryKnown}\label{cor:convex}
     For any horizon $T\in \bbN^+$,  Algorithm~\ref{alg:convex}  calibrated as in Theorem \ref{thm:ExO_adaptive_regret} and tuned with interval size $\sfB$ (which determines  $\eta$) 
     satisfies the following regret guarantees:
     \begin{align*}
         &\textbf{Switching:} \quad \sfB= \tfrac{T}{S} \Longrightarrow R^{\swi} (T,S) \lesssim     d^{\frac{5}{2}} \sqrt{ST}\,, \\ 
         &\textbf{Dynamic:}\quad    \sfB=   \tfrac{T}{S} \vee (d^{\frac{5}{2}}T/\Delta)^{\frac{2}{3}} \Longrightarrow       R^\dyn(T, \Delta, S)\lesssim   R^\swi(T, S) \wedge  d^{\frac{5}{3}}\Delta^{\frac{1}{3}}T^{\frac{2}{3}} \,, \\
         &\textbf{Path-length:} \quad  \sfB=(rd^{\frac{5}{2}}T/P)^{\frac{2}{3}} \Longrightarrow R^\pat(T, P)\lesssim r^{-\frac{1}{3}}  d^{\frac{5}{3}}P^{\frac{1}{3}}T^{\frac{2}{3}} \,. 
     \end{align*}
\end{restatable}
The proofs of Theorem~\ref{thm:ExO_adaptive_regret} and Corollary~\ref{cor:convex} are presented in  Appendix~\ref{appx:ExO_proofs}. By comparing these results to the lower bounds in Section~\ref{sec:lower_bound}, we obtain that for known $S, \Delta$ and $P$, cExO achieves minimax-optimal rates in $T, S$ and $\Delta$, but remains suboptimal in $d$ (for all measures), and potentially for the path-length bound (see Eq.~\eqref{eq:2lowerbound}). This suboptimal dependence on $d$   is unsurprising since even the best known static regret bounds of \cite{fokkema2024online} and \cite{suggala2024second}  suffer from similar dimensional dependence.
Moreover, the gap between cExO's path-length regret  bound and our  minimax  lower bound of  order $d^{\frac{4}{5}} P^{\frac{2}{5}} T^{\frac{3}{5}}$  may stem from either (i) looseness  in the lower bound, or (ii) 
  sub-optimality of cExO, which runs OMD in distribution space rather than directly on the action set.  
  The latter may allow us to bound   path-length regret more  directly and sharply.

To adapt to unknown non-stationarity measures, cExO equipped with the BoB framework yields the upper bounds in Corollary~\ref{cor:convex} with an additional $ d^{\frac{5}{4}}T^{\frac{3}{4}}$ term (see Corollary~\ref{cor:BoB_ExO} in Appendix~\ref{appx:ExO_proofs}). Our path-length regret   of $P^{\frac{1}{3}}T^{\frac{2}{3}}$ and $P^{\frac{1}{3}}T^{\frac{2}{3}}+ T^{\frac{3}{4}}$ for known and unknown $P$, respectively, improves on the $P^{\frac{1}{4}}T^{\frac{3}{4}} $ and $ P^{\frac{1}{2}}T^{\frac{3}{4}}$ rates in \cite{zhao2021bandit} in terms of $T$.

%% file: conclusion.tex
\section{Conclusion}
\label{sec:conclusion}

In this work, we develop and analyze two approaches for non-stationary Bandit Convex Optimization. For strongly convex losses, our polynomial-time TEWA-SE algorithm achieves minimax-optimal dynamic regret w.r.t.\@ $S$ and $\Delta$ without knowing the strong-convexity parameter, but incurs a suboptimal $T^{\frac{3}{4}}$ rate for general convex losses. To address this, we propose a second algorithm, cExO, which achieves minimax-optimality for $S$ and $\Delta$. However, this algorithm is not polynomial-time computable and has an increased dimension dependence. Our matching lower bounds confirm the optimality results, but also reveal 
potentially suboptimal guarantees w.r.t.\@ the path-length $P$. This work highlights a central  open challenge: designing algorithms that are simultaneously minimax-optimal and computationally efficient for general convex losses in non-stationary environments.
A promising stepstone towards this goal is to incorporate second-order information, akin to  the online Newton methods from \cite{fokkema2024online, suggala2024second}  that achieve   state-of-the-art static regret guarantees for adversarial convex bandits. In particular, a restart criterion, similar to the one  in line 15 of  \cite[Algorithm 1]{suggala2021efficient} or   line 11 of  \cite[Algorithm 1]{fokkema2024online},  
may enable tracking capabilities and 
  lead to   improved regret bounds.


%% file: appx_definitions.tex
\section{Definitions}\label{app::definitions}
\begin{definition}\label{def::subgaussian} 
Let $\sigma>0$. A random variable $\xi$ is $\sigma$-sub-Gaussian  if for any $\lambda>0$ we have $\Exp[\exp(\lambda \xi)] \le \exp(\sigma^2 \lambda^2/2)$. 
\end{definition}
\begin{definition}
Let $\alpha>0$. A differentiable function     $f: \bbR^d\to \bbR$ is called $\alpha$-strongly convex, if for $ \bx,\bz\in \bbR^d$,  $f(\bz)\ge f(\bx)+\nabla f(\bx)^\top (\bz-\bx)+\frac{\alpha}{2}\|\bz-\bx\|^2$\,.
\end{definition}
\begin{definition}
  Let $\beta>0$. Function  $f: \bbR^d\to \bbR$ is called $\Lx$-smooth, if it is continuously differentiable and for any $\bx,\bz\in \bbR^d$, 
  $\|\nabla f(\bx)- \nabla f(\bz)\| \le \beta \|\bx - \bz\|$\,. 
\end{definition}
\begin{definition}
    Let $K>0$. Function $f:\bbR^d\to \bbR$ is called $K$-Lipschitz if for any $\bx,\bz\in\bbR^d$, $|f(\bx)-f(\bz)|\le K\|\bx-\bz\|$.
\end{definition}



%% file: appx_reductions.tex
\section{Proof of Proposition~\ref{prop:conversions}}\label{app::conversion}

We start this section by restating the proposition, before detailing its proof.

\propconv*

\begin{proof}[Proof of Proposition \ref{prop:conversions} ]
    The proof follows two steps. First, we state in Lemma~\ref{lem: ada to swi} the conversion between adaptive regret and switching regret. A similar conversion can be found in \cite{daniely2015strongly}, but we detail the proof for completeness. 
Next, we prove in Lemma~\ref{lem: swi to dyn and path simple} that switching regret guarantees for appropriate number of switches convert into dynamic and path-length regret guarantees. 
\end{proof}

In the remainder of this section, we detail the two supporting lemmas and their proof.

\begin{lemma}
\label{lem: ada to swi}
Consider an algorithm that satisfies the adaptive regret guarantees of Proposition~\ref{prop:conversions}, then this algorithm calibrated with interval size $\sfB=\ceil{\tfrac{T}{S}}$ satisfies 
\[R^{\swi} (T, S) \leq 2^{1+\kappa}C S^{1-\kappa} T^\kappa\;.\]
\end{lemma}

\begin{proof}[Proof of Lemma \ref{lem: ada to swi}]
Consider $\sfB=\lceil \frac{T}{S} \rceil$. Let $\bu_{1:T}\in \Theta^T$ be a sequence of  arbitrary comparators with at most $S$ switches. We divide the horizon into intervals of length $\sfB$ (the last interval may be shorter than $\sfB$), and further divide the intervals at the rounds where $\bu_t\neq \bu_{t-1}$. This ensures  each of these intervals is associated with a constant comparator. 
By construction, these intervals are of length $\le \sfB$ and the number of intervals is bounded by $2S$. 
   Hence, we can apply the adaptive regret bound to   each interval to obtain
    \begin{align*}
    R^{\swi}(T, S)& \leq 2S\cdot C\sfB^{\kappa} \leq  2CS\cdot\left(\frac{T}{S}+1\right)^{\kappa}\\
    &= 2C\cdot S^{1-\kappa} T^\kappa \cdot \left(1+\frac{S}{T}\right)^\kappa \leq 2^{1+\kappa}C \cdot   S^{1-\kappa}T^\kappa  \,.
    \end{align*}
\end{proof}


We now prove the conversion between switching regret  and dynamic and path-length regrets.
\begin{lemma}
\label{lem: swi to dyn and path simple}
Consider any fictitious number of switches $S'\in [T]$. Then the dynamic regret of  environments constrained by  $\Delta$ satisfies
\begin{align}
    R^{\dyn}(T, \Delta) \leq R^{\swi}(T, S') + \Delta\ceil{\tfrac{ T }{S'}}\,,\label{eq:lem_R_dyn_S'}
\end{align}
and the path-length regret  satisfies 
\begin{align}\label{eq:lem_R_path_S'}
    R^{\pat}(T, P) \leq R^{\swi}(T, S')+  \tfrac{P}{r}\ceil{\tfrac{T}{S'}}\,.
\end{align}
\end{lemma}
\begin{proof}[Proof of Lemma~\ref{lem: swi to dyn and path simple}]
For both upper bounds, the switching regret term comes from dividing the horizon $[T]$ into $S'$ intervals, denoted by $(I_s)_{s\in [S']}$, each of length at most $\ceil{\tfrac{T}{S'}}$ (defining them precisely is not important for the following arguments). Recall the definition of $R(T, \bu_{1:T})$ from   \eqref{eq:Reg_universal}. For any sequence of actions $\bz_{1:T}\in \Theta^T$  chosen by the given algorithm, and for any arbitrary comparator sequences $\bu_{1:T} \in \Theta^T$ and $\bv_{1:S'} \in \Theta^{S'}$, it holds that
\begin{align}
    R(T, \bu_{1:T}) & = \sum_{t=1}^T \Exp\left[f_t(\bz_t)-f_t(\bu_t)\right]\nonumber \\
    &= \sum_{t=1}^T \sum_{s=1}^{S'} \ind{t\in I_s} \cdot \left(\Exp\left[f_t(\bz_t)-f_t(\bv_s)\right] + \left(f_t(\bv_s)-f_t(\bu_t)\right)\right) \nonumber\\
    &\leq  R^{\swi}(T, S') +  \sum_{s=1}^{S'} \underbrace{\sum_{t\in I_s}  \left(f_t(\bv_s)-f_t(\bu_t)\right)}_{=:V_s} \;,\label{eq:decompose_conversions}
\end{align} 
where the last step holds by the definition of the switching regret. It thus remains to choose a suitable $\bv_s\in \Theta$ and  upper bound the term $V_s$ for $s\in [S']$. We choose a different $\bv_s$ for the proof of the dynamic regret bound vs.\@ that of the path-length regret bound.

{\bf{Dynamic regret.}}
Consider the interval $I_s$ for $s\in [S']$. Let $L_s$ be  its length and $\Delta_s=   \sum_{t\in I_s} \max_{\bz\in \Theta} |f_t(\bz) - f_{t-1}(\bz)|$ be the total variation over this interval.   Then, for any two time steps $t$ and $t'$ in $I_s$ and any $\bz\in\Theta$, it holds that $f_t(\bz)-f_{t'}(\bz)\leq \Delta_s$ by definition of  total variation. Let $\bar f_s$ denote the average of the functions over the interval $I_s$ and define $\bv_s \in \argmin_{\bz\in \Theta} \bar{f}_s(\bz)$, then we have 
\begin{equation*}
    \forall s \in [S'], \quad V_s 
    \leq \sum_{t\in I_s} (f_t(\bv_s)-\bar f_s(\bu_t)+\Delta_s)\leq \Delta_s L_s\,.
\end{equation*} 
Taking the sum over all intervals and using $L_s \leq \ceil{\frac{T}{S'}}$ completes the proof of \eqref{eq:lem_R_dyn_S'}.

{\bf Path-length regret.} This proof proceeds similarly as that for the dynamic regret. 
Consider the interval $I_s$ for some $s\in [S']$, and denote by $L_s$ its length and $P_s=\sum_{t\in I_s}\|\bu_t-\bu_{t-1}\|$ the path-length of the comparator sequence  on this interval. For the proof, 
we construct  $\bv_s\in \Theta$  differently from that in the proof of the dynamic regret. Before detailing the construction  of $\bv_s$, we first 
  define a set of comparators $(\bu_t')_{t\in I_s}\in\Theta^{L_s}$ as follows: for some $\alpha_0\in[0,1]$ and any time $t\in I_s$, we define $\bu_t'$ to satisfy  $\bv_s = \alpha_0 \bu_t+(1-\alpha_0) \bu_t'$. Using this and  by the convexity and boundedness of $f_t$, we can bound that
\[f_t(\bv_s) \leq f_t(\bu_t) + (1-\alpha_0)(f_t(\bu_t')-f_t(\bu_t))\leq f_t(\bu_t)+2(1-\alpha_0)\,.\] 
We then proceed by  choosing a suitable  $\bv_s$ and  $\alpha_0$ to make this bound depend on the path-length. Since the path-length is 
$P_s$, there exists an $\ell_2$-ball of radius $\frac{P_s}{2}$ that contains all the comparators $(\bu_t)_{t\in I_s}$, and its center $\bc_u$ lies in the feasible domain $\Theta$. By assumption (as we stated in Section \ref{sec:intro}), there also exists a ball with radius $r$ and center $\bc_r$ within the domain. We can thus construct $\bv_s$ to satisfy
\begin{align}
&\bv_s = \alpha_0 \bc_u +(1-\alpha_0)\bc_r \in \Theta\,,\qquad \text{which yields} \quad  \bu'_t=\bc_r +\frac{\alpha_0}{1-\alpha_0}(\bc_u-\bu_t)\,, \label{eq:ut_prime_construction}
\end{align}
where $\bv_s\in\Theta$ due to the convexity of the domain. 
Our goal is then to choose $\alpha_0$ as large as possible (to make $1-\alpha_0$ small) such that all the comparators $(\bu_t')_{t\in I_s}$ belong to $\Theta$. Eq.~\eqref{eq:ut_prime_construction} implies that 
\begin{align*}
    \|\bu'_t-\bc_r\|=\frac{\alpha_0}{1-\alpha_0}\|\bc_u-\bu_t\|\leq \frac{\alpha_0}{1-\alpha_0}\cdot \frac{P_s}{2}=\frac{\alpha_0 P_s}{2(1-\alpha_0)}\,,
\end{align*}
which by definition of the $r$-ball guarantees that $\bu_t'\in \Theta$ as long as $\frac{\alpha_0 P_s}{2(1-\alpha_0)} \leq r$. To satisfy this condition, we can thus pick $\alpha_0 = \frac{2r}{P_s+2r}$, which guarantees by construction that
\begin{equation*}
    \forall t \in I_s: \quad f_t(\bv_s)\leq f_t(\bu_t)+\frac{2P_s}{P_s+2r} \leq  f_t(\bu_t)+\frac{P_s}{r}.
\end{equation*}
The desired bound on $V_s$ in \eqref{eq:decompose_conversions} directly follows. 
The final result \eqref{eq:lem_R_path_S'} then comes by summation over all intervals $(I_s)_{s\in [S']}$.
\end{proof}

%% file: appx_details_TEWA.tex
\section{Details and proofs for  TEWA-SE }\label{appx:TEWA_details_and_proofs}
In this appendix, we provide additional details on TEWA-SE in Section~\ref{sec:details_ada_TEWA} and establish its theoretical guarantees in  Sections~\ref{sec:regret_GC_interval}--\ref{appx:TEWA_BoB}.  
We present the proof of Theorem~\ref{thm:interval_regret} in Section~\ref{sec:appx_TEWA_ada}, followed by the supporting lemmas  in  Sections~\ref{sec:regret_any_interval} and \ref{sec:regret_GC_interval}. We then provide   the proof of Corollary~\ref{cor:TEWA_R_swi_known} in Section~\ref{sec:appx_TEWA_cor}, and the parameter-free guarantees in Section \ref{appx:TEWA_BoB}.

\subsection{Additional details on TEWA-SE}
\label{sec:details_ada_TEWA}
As we described in Section \ref{sec:AdaptiveReg}, TEWA-SE  handles non-stationary environments by employing the Geometric Covering (GC) scheme from \cite{daniely2015strongly} to schedule experts across different time intervals. Additionally,   TEWA-SE assigns an exponential grid of learning rates to  the  multiple experts covering each  GC interval, to adapt to  the  curvature of the loss functions.  We first invoke the definition of GC intervals from \cite{daniely2015strongly}.
\begin{definition}[Geometric Covering (GC) intervals \cite{daniely2015strongly}]\label{def:GC_intervals}
For $k\in \bbN $, define the set of intervals
\begin{align}
 \mc{I}_k=\left\lbrace [i\cdot 2^k, (i+1)\cdot 2^k-1]: i\in \bbN^+\right\rbrace,
\end{align}
that is, $ \mc{I}_k$ is a partition of $\bbN^+\setminus [2^k-1]$ into intervals of length $2^k$. Then we call 
$ \mc{I}=\bigcup_{k \in \bbN}\mc{I}_k$ 
the set of Geometric Covering (GC) intervals.
\end{definition}
For any interval length $L\in \bbN^+$, we also define  the exponential grid of learning rates as
\begin{align}
    \mc{S}(L) = \left\lbrace \frac{2^{-i}}{5GD}  : i\in \left\{0,1, \dots, \ceil{\tfrac{1}{2}\log_2 L}\right\}\right\rbrace 
    \,,\label{eq:set_S(t)}
\end{align}
where  $G$ is the uniform upper bound \eqref{eq:def_hat_G} on $\|\bg_t\|$, and $D$ is the diameter of the feasible set $\Theta$. For each given GC interval $I=[r,s]\in \mc{I}$, TEWA-SE instantiates multiple experts  in round $r$, each assigned a distinct learning rate $\eta\in \mc{S}(|I|)$ and surrogate loss $\ell_t^\eta$ as defined in \eqref{eq:surrogate_loss}. It    removes these experts after round $s$. This scheduling scheme ensures at least one expert covering $I$ effectively minimizes the linearized regret $\sum_{t\in I}\left\langle \Exp[\bg_t| \bx_t, \Lambda_T], \bx_t-\bu\right\rangle$ associated with $\hat{f}_t$ on the interval $I$ (Lemma \ref{lem:lin_reg_GC_interval}),  ultimately  yielding the regret guarantees in Theorem \ref{thm:interval_regret} and Corollary \ref{cor:TEWA_R_swi_known}.

\paragraph{Polylogarithmic computational complexity}For $t\in \bbN^+,$ we use $\mc{C}_t=\{I \in \mc{I}: \; t\in I\}$ to denote the set of GC intervals covering time $t$. 
From Definition \ref{def:GC_intervals} it is easy to verify that  $|\mc{C}_t|=  1+ \floor{\log_2t}$. The longest interval in $\mc{C}_t$ has length at most $ t$, which is associated with at most $|\mc{S}(t)|= 1+ \ceil{\tfrac{1}{2}\log_2 t} $ experts. With $\mc{A}_t=\{E(I, \eta): t\in I\}$ representing the set of  experts active in round $t$, the number of active experts in round $t$, denoted by $n_t=|\mc{A}_t|$ in Algorithm \ref{alg:metagrad_lean}, satisfies
\begin{align}\label{eq:nt_bound}
    n_t\le   (1+\floor{\log_2 t}) \cdot  \left(1+ \ceil{\tfrac{1}{2}\log_2 t}\right)\,.
\end{align}
This ensures that the computational complexity of TEWA-SE is only $\mc{O}(\log^2 T)$ per round.

\paragraph{Tilted Exponentially Weighted Average} In each round $t$, TEWA-SE aggregates the actions proposed by the active experts $E(I, \eta)\in \mc{A}_t$
using exponential weights, tilted by their respective learning rates, by computing
\begin{align}
    \bx_t = \frac{\sum_{E(I, \eta)\in \mc{A}_t } \eta\exp(-L_{t-1, I}^{\eta}) \bx_{t, I}^{\eta} }{\sum_{E(\tilde{I}, \tilde{\eta})\in \mc{A}_t }\tilde{\eta} \exp(-L_{t-1, \tilde{I}}^{\tilde{\eta}}) }\,,\label{eq:TEWA_update_rule}
\end{align}
where for $I=[r,s]$ and $t\in [r+1, s]$,  $ L_{t-1,I}^{\eta} = \sum_{\tau =r}^{t-1} \ell_\tau^{\eta} (\bx_{\tau, I}^{\eta})$ represents the cumulative  surrogate loss accrued  by expert $E(I, \eta)$ over the interval $[r,t-1]$. Note that \eqref{eq:TEWA_update_rule} is equivalent to line \ref{line:update_xt_lean} of Algorithm \ref{alg:metagrad_lean}, rewritten with  notation better suited  for our proof. 

%% file: appx_upper_bound_TEWA.tex
In what follows, we prove some theoretical guarantees for TEWA-SE. 

\subsection{Proof of Theorem \ref{thm:interval_regret}}\label{sec:appx_TEWA_ada}
In this section, we first restate Theorem 1 and provide its complete proof, which relies on several supporting lemmas. For clarity of exposition, we defer the statements and proofs of these supporting lemmas to the following sections.
\thmTEWAada*
\begin{proof}[Proof of Theorem \ref{thm:interval_regret}]
 We  prove \eqref{eq:thm_adaptive_regret} for the general convex case and \eqref{eq:thm_adaptive_regret_sc} for the strongly-convex case similarly. To bound $R^\ada(\sfB, T)$, we will uniformly bound $ \sum_{t=p}^q\Exp[f_t(\bz_t) - f_t(\bu)]$ across all   comparators $\bu\in \Theta$ and intervals $[p, q]$ shorter than $\sfB$. \\ 
\textbf{Common setup:}  
    Invoking the event $\Lambda_T = \big\lbrace|\xi_t| \le 2\sigma \sqrt{\log (T+1)}, \; \forall t\in [T]\big\rbrace$  defined above  \eqref{eq:def_hat_G},  since $\{\xi_t\}_{t=1}^T$ are $\sigma$-sub-Gaussian, we have 
\begin{align}\label{eq:bound_P_Lambda_T_c}
    \Prob(\Lambda_T^c) \le \sum_{t=1}^T   \Prob\left(|\xi_t| > 2\sigma \sqrt{\log (T+1)}\right) \le 2\sum_{t=2}^{T+1} T^{-2} = 2T^{-1}.
\end{align}
  By the law of total expectation we can write for any $\bu\in \Theta$,
\begin{align}\label{eq:regret_condition_LambdaT}
    \sum_{t=p}^q \Exp[f_t(\bz_t) - f_t(\bu)] 
    &= \sum_{t=p}^q \Exp[f_t(\bz_t) - f_t(\bu)\mid \Lambda_T] \Prob(\Lambda_T)+\sum_{t=p}^q \Exp[\underbrace{f_t(\bz_t) - f_t(\bu)}_{\le 2}\mid \Lambda_T^c] \underbrace{\Prob(\Lambda_T^c)}_{\le 2T^{-1}} \nonumber\\
    &\le  \sum_{t=p}^q \Exp[f_t(\bz_t) - f_t(\bu)\mid \Lambda_T] +4 \,.
\end{align}
To bound the first term in the last display, we consider the following decomposition
\begin{align}
  &\sum_{t=p}^q  \Exp[f_t(\bz_t) - f_t(\bu)\mid \Lambda_T] \nonumber\\
  &=\underbrace{\sum_{t=p}^q \Exp[f_t(\bz_t) - f_t(\bx_t)\mid \Lambda_T]}_{\text{term I}}
    + \underbrace{\sum_{t=p}^q\Exp[f_t(\bx_t) - \hat{f}_t(\bx_t)\mid \Lambda_T]}_{\text{term II}}\; +\nonumber\\
    &\quad\, \underbrace{\sum_{t=p}^q\Exp[ \hat{f}_t(\bx_t) - \hat{f}_t(\bu)\mid \Lambda_T]}_{\text{term III}}
    + \underbrace{\sum_{t=p}^q\Exp[\hat{f}_t(\bu) - f_t(\bu)\mid \Lambda_T]}_{\text{term IV}}\,.\label{eq:termI_IV_decompose}
\end{align}
%
 Since $f_t$ is convex, by Jensen's inequality  we obtain that term II is negative (c.f.\@  \cite[Lemma A.2 (ii)]{akhavan2020exploiting}). 
In what follows, we bound  terms I, III and IV in this decomposition separately for the general convex case and the strongly-convex case.

\textbf{General convex and Lipschitz case:} Recall that $(\bzeta_t)_{t=1}^T$ denote uniform samples from the unit sphere $\partial \mathbb{B}^d$, and  $\tilde{\bzeta}$ denotes a uniform sample from the unit ball $\mathbb{B}^d$, while $\hat{f}_t(\bx)=\Exp[f_t(\bx+h\tilde{\bzeta})]\;\forall \bx\in\tilde{\Theta}$. Since $(f_t)_{t=1}^T$ are $K$-Lipschitz,   $\|\bzeta_t\|=1$, and  $ \Exp[\|\tilde{\bzeta}\|]\le 1 $, we can bound term I and term IV by 
\begin{align}
   & \text{term I}= \sum_{t=p}^q\Exp\Big[\Exp[f_t(\bx_t + h \bzeta_t)|\bx_t] - f_t(\bx_t) \mid \Lambda_T\Big]
    \le  K(q-p+1)h,\\
   & \text{term IV}=\sum_{t=p}^q\Exp\Big[\Exp[f_t(\bu+ h \tilde{\bzeta})] - f_t(\bu) \mid \Lambda_T\Big] \le  K(q-p+1) h.
\end{align}
To bound term III,  
recall that $\bg_t$ denotes the gradient estimate of $\hat{f}_t$ at $\bx_t$. We use the convexity of $\hat{f}_t$ and  apply Lemma~\ref{lem:lin_reg_any_interval} to obtain that for any $\bu\in \Theta$, 
\begin{align}\label{eq:termIII}
 \text{term III}   &   \le \Exp\left[\sum_{t=p}^q \left\langle\Exp[\bg_t|\bx_t, \Lambda_T], \bx_t-\bu\right \rangle \big|\, \Lambda_T\right] \nonumber\\
 & \le \Exp\Bigg[10  GD a_{p,q}b_{p,q}
+ 3 G\sqrt{a_{p,q}b_{p,q}}\sqrt{\sum_{t=p}^q \|\bx_t -\bu\|^2}  \;\big|\, \Lambda_T\Bigg]\nonumber\\
& \le  10  GD a_{p,q}b_{p,q}
+ 3 GD\sqrt{a_{p,q}b_{p,q}}\sqrt{q-p+1}  \,,
\end{align}
where all constants are explicit in the statement of the lemma. 
By combining the bounds for all  four terms in \eqref{eq:termI_IV_decompose} with  \eqref{eq:regret_condition_LambdaT}, and using $h = \min\big(\sqrt{d} \sfB^{-\frac{1}{4}}, r\big)$, $ G = \tfrac{d}{h}(1+2\sigma\sqrt{\log(T+1)})$, $a_{p,q} = \tfrac{1}{2}+2\log(2q) +\tfrac{1}{2} \log(q-p+1) \le 6\log(T+1)$, and  $b_{p,q}= 2\ceil{\log_2(q-p+2)}\le 6\log(\sfB+1)$,  we  establish that 
\begin{align}
R^\ada(\sfB, T)
    &\le 10  GD a_{p,q}b_{p,q}+  3GD\sqrt{a_{p,q}b_{p,q}} \sqrt{q-p+1 }+  2K (q-p+1)h+4\nonumber\\
    & \le \begin{cases}
C\sqrt{d}\sfB^{\frac{3}{4}}+4& \text{if } h=\sqrt{d}\sfB^{-\frac{1}{4}}\\
C_1d^2 + C_2 d + 4 &\text{if }h=r
\end{cases}\nonumber\\
    & \le C\sqrt{d}\sfB^{\frac{3}{4}}+ C_1d^2+ C_2d+4\,,\label{eq:R_ada_derivation}
    \end{align}
      where
      $C, C_1, C_2>0$ are polylogarithmic  in  $  T$, independent of $d$, defined with $M_T= 1+2\sigma \sqrt{\log(T+1)}$ and     $N_{ T,\sfB} = \log(T+1)\log(\sfB+1)$ as 
      \begin{align}
C& =18DM_T\big(\sqrt{N_{T,\sfB}} +20 N_{T,\sfB}\sfB^{-\frac{1}{2}} \big)  + 2K\label{eq:C_def}\\   C_1&=\big(18DM_T\sqrt{N_{T,\sfB}} + 2K\big)/r^3\\
      C_2&=360DM_TN_{T,\sfB}/r\,. \label{eq:C2_def}
      \end{align}
 This concludes the proof of \eqref{eq:thm_adaptive_regret}.

\textbf{Strongly-convex   and  smooth case:}  
Due to the $\beta$-smoothness of $f_t$ and the fact that $\Exp[\bzeta_t]=\Exp[\tilde{\bzeta}]=\bzero$ and $\Exp[\|\tilde{\bzeta}\|^2]\le   \Exp[\|\bzeta_t\|^2]=1$, we can bound term I and term IV each by $\frac{\beta}{2}(q-p+1)h^2$.
  When the $f_t$'s are strongly-convex, we can   derive a tighter bound on term III  than that in \eqref{eq:termIII}  
   by restricting the comparator $\bu$ to the clipped domain $\tilde{\Theta}$ 
   and using the fact that when   $f_t$ is $\alpha$-strongly convex  on $\Theta$,  $\hat{f}_t$  is   $\alpha$-strongly convex on $\tilde{\Theta}$ (c.f.\@ \cite[Lemma A.3]{akhavan2020exploiting}).  That is, we have for any $\bu\in \tilde{\Theta}$, 
   \begin{align}
     \text{term III}
   &   \le \Exp\left[\sum_{t=p}^q \left\langle\Exp[\bg_t|\bx_t, \Lambda_T], \bx_t-\bu\right \rangle -\frac{\alpha}{2}\sum_{t=p}^q\|\bx_t-\bu\|^2\,\big|\, \Lambda_T\right] \nonumber\\
 & \le \Exp\Bigg[10  GD a_{p,q}b_{p,q}
+\underbrace{3 G\sqrt{a_{p,q}b_{p,q}}\sqrt{\sum_{t=p}^q \|\bx_t -\bu\|^2}- \frac{\alpha}{2}\sum_{t=p}^q \|\bx_t -\bu\|^2 }_{=:\delta}\big| \Lambda_T\Bigg]\nonumber\\
&\le \left( 10  GD+ \tfrac{18}{\alpha} G^2\right)a_{p,q}b_{p,q}\,,\label{eq:term_III_sc}
    \end{align} 
    where the last inequality holds because term $\delta$ is uniformly bounded as follows:
\begin{align}
\delta \leq \begin{cases}
\frac{18}{\alpha} G^2a_{p,q}b_{p,q} & \text{if } \sqrt{\sum_{t=p}^q\|\bx_t-\bu\|^2} \leq \frac{6}{\alpha} G\sqrt{a_{p,q}b_{p,q}}\,, \\
0 & \text{otherwise}.
\end{cases}\nonumber
\end{align}
Combining \eqref{eq:term_III_sc} with \eqref{eq:regret_condition_LambdaT}--\eqref{eq:termI_IV_decompose} and simplifying yields for any $\bu\in \tilde{\Theta}$,
\begin{align}
    \sum_{t=p}^q  \Exp[f_t(\bz_t) - f_t(\bu) ] 
    &\le \left( 10  GD+ \tfrac{18}{\alpha} G^2\right)a_{p,q}b_{p,q}+  \beta (q-p+1)h^2+4\,. 
\label{eq:regret_without_clipping}
    \end{align} 
    The final step is to handle the case where the comparator $\bu\in \Theta\setminus\tilde{\Theta}$. Consider the worst case when the comparator is  $ \bu^* \in \argmin_{\bu\in\Theta}\sum_{t=p}^q f_t(\bu)$ with $\bu^*\in \Theta\setminus\tilde{\Theta}$.  
%
 Let $\tilde{\bu}^*=\Pi_{\tilde{\Theta}}(\bu^*)$.  If   $\argmin_{\bx\in \bbR^d}f_t(\bx) \in \Theta\; \forall t\in[T]$, then by the $\beta$-smoothness of the $f_t$'s we have 
    \begin{align}
        \sum_{t=p}^q f_t(\tilde{\bu}^*) - f_t(\bu^*) \le \sum_{t=p}^q\Big[\Big\langle \underbrace{\nabla f_t(\bu^*)}_{=0}, \tilde{\bu}^*-\bu^*\Big\rangle  + \frac{\beta}{2}\underbrace{\|\tilde{\bu}^*-\bu^*\|^2}_{\le h^2} \Big] =\frac{\beta}{2}(q-p+1)h^2
.\label{eq:regret_diff_due_to_clipping}
\end{align}
Combining \eqref{eq:regret_without_clipping} with \eqref{eq:regret_diff_due_to_clipping} yields
\begin{align}\label{eq:R_ada_derivation_SC}
    R^\ada(\sfB, T)
    &\le \left( 10  GD+ \tfrac{18}{\alpha} G^2\right)a_{p,q}b_{p,q}+  \tfrac{3}{2}\beta (q-p+1)h^2+4\nonumber\\
    & \le \begin{cases}
C'  d\sqrt{\sfB}+4& \text{if } h=\sqrt{d}\sfB^{-\frac{1}{4}}\\
C_1'd^2 + C_2' d + 4 &\text{if }h=r
\end{cases}\nonumber\\
&\le C'  \tfrac{d}{\alpha}\sqrt{\sfB}+C_1' \tfrac{1}{\alpha} d^2+ C_2' d +4\,,
\end{align} 
where $C', C_1', C_2'>0$ are polylogarithmic in  $T$ and $ \sfB$, independent of $d$, defined as
\begin{align}
C'&=72 \left(9 M_T +5\alpha D \sfB^{-\frac{1}{4}}\right)   {M}_T N_{T,\sfB}+ \tfrac{3}{2}\beta \alpha \label{eq:C_prime_def}\\
C_1'&= \big(648 M_T^2 N_{T, \sfB} +\beta\alpha\big)/r^2\\
C_2'&= 360DM_TN_{T, \sfB}/r\,.\label{eq:C2_prime_def}
\end{align}
This concludes the proof of \eqref{eq:thm_adaptive_regret_sc}.
   \end{proof}
The proof above crucially relies on Lemma \ref{lem:lin_reg_any_interval}, which we state and prove in the following section. 



  \subsection{Upper bounds on linearized regret  }\label{sec:regret_any_interval} 
  Lemma \ref{lem:lin_reg_any_interval} establishes an upper bound on the linearized regret associated with the smoothed loss $\hat{f}_t$ for any arbitrary interval $I=[p,q]\subseteq [1, T]$. This result builds on two key components:  Lemma \ref{lem:daniely_partition}, which characterizes how a given arbitrary interval is covered by a sequence of GC intervals, and Lemma \ref{lem:lin_reg_GC_interval}, which provides an upper bound  on the linearized regret for each GC interval $I\in \mc{I}$. For clarity, we first present and prove Lemma \ref{lem:lin_reg_any_interval}, then proceed to detail the supporting  Lemmas  \ref{lem:daniely_partition} and \ref{lem:lin_reg_GC_interval}.

\begin{lemma}[Linearized regret on an arbitrary   interval]\label{lem:lin_reg_any_interval} For an arbitrary interval $I=[p,q]\subseteq[1, T]$, Algorithm~\ref{alg:metagrad_lean} satisfies for  all $\bu\in \Theta$,
\begin{align}
    \sum_{t=p}^q\left\langle \Exp[\bg_t|\bx_t, \Lambda_T], \bx_t-\bu\right\rangle\le 10  GD a_{p,q}b_{p,q}
+ 3 G\sqrt{a_{p,q}b_{p,q}}\sqrt{\sum_{t=p}^q \|\bx_t -\bu\|^2},\label{eq:lem_lin_regret_any_interval}
\end{align}where $a_{p,q} = \tfrac{1}{2}+2\log(2q) + \tfrac{1}{2}\log(q-p+1)$ and $b_{p,q}= 2\ceil{\log_2(q-p+2)}$.
\end{lemma}
\begin{proof}[Proof of Lemma \ref{lem:lin_reg_any_interval}]This proof follows similar arguments to those used in proving the first part of \cite[Theorem 2]{zhang2021dualadaptivity}.
 To begin, according to Lemma~\ref{lem:daniely_partition},  any arbitrary interval $I=[p, q]\subseteq[1, T]$ can be covered by two sequences of consecutive  and disjoint    GC intervals, denoted by $ I_{-m}, \dots, I_0\in \mc{I}$ and $I_1, \dots, I_n\in \mc{I}$, where $n,m\in \bbN^+$ with $n\le \ceil{\log_2(q-p+2)}$ and $m+1\le \ceil{\log_2(q-p+2)}$. Note that negative indices correspond to GC intervals that precede $I_0$, while positive indices correspond to intervals that follow it. The indices indicate temporal ordering and are unrelated to the length of the intervals.

By applying the  linearized regret bound from Lemma~\ref{lem:lin_reg_GC_interval} to each  GC interval, and noticing that $a_{r,s}\le a_{p,q}$ for any subinterval $[r,s]\subseteq [p,q]$ (as evident from the definition of $a_{p,q}$ in \eqref{eq:lem_lin_regret_any_interval}), we establish for all $\bu\in \Theta$,
 \begin{align}
&\sum_{t=p}^q\left\langle \Exp[\bg_t|\bx_t, \Lambda_T], \bx_t-\bu\right\rangle = 
\sum_{i=-m}^n\sum_{t\in I_i}\left\langle \Exp[\bg_t|\bx_t, \Lambda_T], \bx_t-\bu\right\rangle\nonumber\\
&\le  \sum_{i=-m}^n \left[3 G\sqrt{a_{p,q}\sum_{t\in I_i} \|\bx_t-\bu\|^2} + 10 GDa_{p,q}\right]\nonumber\\
& = 10  GDa_{p,q}(n+m+1) + 3 G\sqrt{a_{p,q}}\sum_{i=-m}^n\sqrt{\sum_{t\in I_i} \|\bx_t -\bu\|^2}\nonumber\\
& \le 10  GDa_{p,q}(n+m+1) + 3 G\sqrt{a_{p,q}}\sqrt{(n+m+1)\sum_{i=-m}^n\sum_{t\in I_i} \|\bx_t -\bu\|^2}\nonumber\\
& \le 10  GD a_{p,q}b_{p,q}
+ 3 G\sqrt{a_{p,q}b_{p,q}}\sqrt{\sum_{t=p}^q \|\bx_t -\bu\|^2}\,,
 \end{align}
 where the last step uses $ n+m+1\le 2\ceil{\log_2(q-p+2)}=:b_{p,q}$.
\end{proof}

We now present Lemmas    \ref{lem:daniely_partition} and \ref{lem:lin_reg_GC_interval} which we used to prove Lemma \ref{lem:lin_reg_any_interval} above.
\begin{lemma}[Covering property of GC intervals]\label{lem:daniely_partition}
 Any arbitrary interval $I=[p,q]\subseteq \bbN^+$   can be partitioned into two finite sequences of  consecutive and disjoint GC intervals, denoted by  $I_{-m}, \dots, I_0\in \mc{I}$ and $I_1, \dots, I_n \in \mc{I}$, where $I= \bigcup_{i=-m}^nI_i\,$, such that 
  \begin{align}\label{eq:GC_property}
      \frac{|I_{-i}|}{ |I_{-i+1}|} \le \tfrac{1}{2}\quad \forall i\ge 1, \quad \text{and}\quad
      \frac{|I_{i}|}{ |I_{i-1}|} \le \tfrac{1}{2}, \quad \forall i\ge 2\,,
  \end{align}
  with 
  \begin{align}\label{eq:n_upper_bound}
      n\le \ceil{\log_2(q-p+2)},\quad \text{and}\quad m+1\le \ceil{\log_2(q-p+2)}.
  \end{align}
\end{lemma}
\begin{proof}[Proof of Lemma \ref{lem:daniely_partition}]
  Eq.~\eqref{eq:GC_property} directly comes from  \cite[Lemma 1.2]{daniely2015strongly}. To prove \eqref{eq:n_upper_bound}, suppose for contradiction   $n> \ceil{\log_2(q-p+2)}$, then we have 
 \begin{align}
     \sum_{i=1}^n|I_i|\ge \sum_{i=1}^n 2^{i-1} = 2^n-1 >  q-p+1=|I|\,,
 \end{align}
 contradicting the fact that $\bigcup_{i=-m}^n I_i=I$. 
By the same reasoning, we have $m+1\le \ceil{\log_2(q-p+2)}$.
\end{proof}

\begin{lemma}[Linearized regret on a   GC interval]\label{lem:lin_reg_GC_interval}For any GC interval $I=[r,s]\in \mc{I}$, Algorithm~\ref{alg:metagrad_lean} satisfies   for all $ \bu\in \Theta$,
\begin{align}
    \sum_{t=r}^s\left\langle \Exp[\bg_t | \bx_t, \Lambda_T], \bx_t- \bu\right\rangle \le 3 G\sqrt{a_{r,s}\sum_{t=r}^s \|\bx_t-\bu\|^2} + 10 GDa_{r,s}, 
\end{align}
where $a_{r,s} = \tfrac{1}{2}+2\log(2s) + \tfrac{1}{2}\log(s-r+1)$.
\end{lemma}

\begin{proof}[Proof of Lemma \ref{lem:lin_reg_GC_interval}]
This proof is similar to that of \cite[Lemma 12]{zhang2021dualadaptivity}.    For any GC interval  $ I=[r, s]\in \mc{I}$ and learning rate $\eta\in \mc{S}(s-r+1)$, we can  apply the  definition of  surrogate loss $\ell_t^\eta$ from \eqref{eq:surrogate_loss}, noticing that $\ell_t^\eta(\bx_t)=0$, to obtain for all  $\bu\in \Theta$,
\begin{align}
&\sum_{t=r}^s\eta\left\langle \Exp[\bg_t| \bx_t,\Lambda_T], \bx_t-\bu\right\rangle -\sum_{t=r}^s\eta^2  G^2  \|\bx_t-\bu\|^2\nonumber\\
&=
\sum_{t=r}^s  \Exp[ -\ell_t^\eta (\bu)\mid \bx_t, \Lambda_T] \nonumber\\
&=\underbrace{\sum_{t=r}^s\Exp\left[\ell_t^\eta (\bx_t) - \ell_t^\eta (\bx_{t, I}^\eta)\mid \bx_t, \Lambda_T\right]}_{\text{meta-regret}\le 2\log(2s)}
+\underbrace{\sum_{t=r}^s\Exp\left[\ell_t^\eta (\bx_{t,I}^\eta) - \ell_t^\eta (\bu) \mid \bx_t, \Lambda_T\right]}_{\text{expert-regret}\le \frac{1}{2}+ \frac{1}{2}\log(s-r+1)}\nonumber\\
& \le 2\log(2s)+ \tfrac{1}{2}+ \tfrac{1}{2}\log(s-r+1)=: a_{r,s}
\,,\label{eq:TEWA_regret_decompose}
\end{align}
where the last step applies   the  upper bound on  the expert-regret established in Lemma~\ref{lem:base_regret} and  the upper bound on the meta-regret  in Lemma~\ref{lem:meta_regret}, both of which we defer to Section \ref{sec:regret_GC_interval}. 
Eq.~\eqref{eq:TEWA_regret_decompose} can be rearranged into
\begin{align}
    \sum_{t=r}^s \left\langle \Exp[\bg_t|\bx_t, \Lambda_T], \bx_t-\bu\right\rangle 
    &\le \frac{a_{r,s}}{\eta} +\eta G^2\sum_{t=r}^s \|\bx_t-\bu\|^2\,.\label{eq:interval_linear_regret}
\end{align} 
The optimal value of $\eta$ that minimizes the RHS of \eqref{eq:interval_linear_regret} is 
\begin{align}
\eta^*=\sqrt{\frac{a_{r,s}}{ G^2\sum_{t=r}^s\|\bx_t-\bu\|^2} }.
\end{align}
Note that since $a_{r,s}\ge \tfrac{1}{2} $, $ \eta^*\ge \frac{1}{ GD\sqrt{2(s-r+1)}}$ for all   $\bx\in \Theta$.  
The next step is to  select a value  $\eta$ from the set $\mc{S}(s-r+1) = \left\lbrace \frac{2^{-i}}{5 GD}  : i\in \big\{0,1, \dots, \ceil{\frac{1}{2}\log_2 (s-r+1)}\big\}\right\rbrace$ that best approximates $\eta^*$. Two cases arises:
\begin{enumerate}[label=\roman*)]
    \item If $\eta^*\le \frac{1}{5 GD}$, there must exist an $\eta\in \mc{S}(s-r+1)$ such that 
    $ \frac{\eta^*}{2}\le\eta\le \eta^*$. Substituting this choice of $\eta$ into \eqref{eq:interval_linear_regret} gives
    \begin{align}
        \sum_{t=r}^s \left\langle \Exp[\bg_t|\bx_t, \Lambda_T], \bx_t-\bu\right\rangle 
    &\le \frac{2a_{r,s}}{\eta^*} +\eta^*  G^2\sum_{t=r}^s \|\bx_t-\bu\|^2=3  G
\sqrt{a_{r,s} \sum_{t=r}^s \|\bx_t-\bu\|^2}.    \label{eq:UB1} \end{align}
    \item If $\eta^*>  \frac{1}{5 GD}$, then the best choice of $\eta\in \mc{S}(s-r+1)$ is $\eta=\frac{1}{5 GD}$, which leads to 
    \begin{align}
         \sum_{t=r}^s \left\langle \Exp[\bg_t|\bx_t, \Lambda_T], \bx_t-\bu\right\rangle 
    &\le  a_{r,s}\cdot 5 GD \cdot 2 = 10 GDa_{r,s}\,.\label{eq:UB2}
    \end{align}
\end{enumerate}
Combining \eqref{eq:UB1}--\eqref{eq:UB2} concludes the proof.\end{proof}

The proof of Lemma \ref{lem:lin_reg_GC_interval} above relied on the upper bounds on the expert-regret and meta-regret  from Lemmas~\ref{lem:base_regret} and \ref{lem:meta_regret}. We present and prove these lemmas in the following section.


\subsection{Upper bounds on expert-regret and meta-regret}\label{sec:regret_GC_interval}


\begin{lemma}[Expert-regret]\label{lem:base_regret}
 For any GC interval $I=[r,s]\in \mc{I}$ and learning rate $\eta\in \mc{S}(s-r+1)$, Algorithm~\ref{alg:metagrad_lean} satisfies for all $\bu\in\tilde{\Theta}$,
 \begin{align}
     \sum_{t=r}^s\Exp\left[ \ell_t^\eta(\bx_{t, I}^\eta)-   \ell_t^\eta(\bu)\mid  \bx_t,\bx_{t,I}^{\eta},\Lambda_T \right]\le \tfrac{1}{2}+\tfrac{1}{2}\log(s-r+1)\,.
 \end{align}
 \end{lemma}

\begin{proof}[Proof of Lemma \ref{lem:base_regret}] 
The proof follows standard convergence analysis  of  projected online gradient descent   for strongly convex objective functions, see e.g., \cite[Theorem 1]{hazan2007logarithmic}. For any time step $t\in I$,  the surrogate loss $\ell_t^\eta$  associated with the expert with learning rate $\eta$ and lifetime $I=[r, s]$ serves as our strongly-convex objective function. 
   By applying the definition of $\ell_t^\eta$, we have for all $\bx\in\com$,
    \begin{align}
       \Exp\left[\|\nabla \ell_t^\eta(\bx)\|^2 \mid \bx_t, \Lambda_T\right]
       &= 
       \Exp\left[ \|\eta \bg_t + 2\eta^2 G^2 (\bx-\bx_t)\|^2\mid \bx_t, \Lambda_T\right]\nonumber\\
       & \le \Exp\left[ \left(\|\eta \bg_t\| + \|2\eta^2 G^2 (\bx-\bx_t)\|\right)^2\mid \bx_t, \Lambda_T\right]
       &
       \le
       (G')^2\,,\label{eq:UB_norm_grad_ell}
    \end{align}
     where we introduced $G' = \eta G +2\eta^2 G^2D$. By the update rule of our projected online gradient descent with step size $\mu_t$ (line~\ref{line:base_update_lean} of Algorithm~\ref{alg:ogd_lean}), we have for all $\bu\in \tilde{\Theta}$, 
    \begin{align*}
        \|\bx_{t+1, I}^{\eta}-\bu\|^2& =\big \| \Pi_{\tilde{\Theta} }\big(\bx_{t, I}^\eta - \mu_t \nabla \ell_t^\eta (\bx_{t, I}^\eta)\big) - \bu\big\|^2\nonumber\\
        &\le \big\| \bx_{t, I}^\eta - \mu_t \nabla \ell_t^\eta (\bx_{t, I}^\eta) - \bu\big\|^2\nonumber\\
        &=\|\bx_{t, I}^\eta -\bu\|^2
        +\mu_t^2\|\nabla \ell_{t}^\eta (\bx_{t, I}^\eta )\|^2 - 2\mu_t(\bx_{t, I}^\eta -\bu)^\top \nabla\ell_t^\eta (\bx_{t,I}^{\eta})\,,
    \end{align*}
    which can be rearranged into
    \begin{align}
        2(\bx_{t,I}^\eta -\bu)^\top \nabla \ell_t^\eta (\bx_{t, I}^\eta) \le \frac{\|\bx_{t, I}^\eta -\bu\|^2-\|\bx_{t+1, I}^{\eta}-\bu\|^2
        }{\mu_t}+\mu_t\|\nabla \ell_{t}^\eta (\bx_{t, I}^\eta )\|^2\,. \label{eq:cross_term}
    \end{align}
Define shorthand $\lambda\equiv 2\eta^2G^2$ and recall $\mu_t= 1/(\lambda (t-r+1))$, then 
Eq.~\eqref{eq:cross_term} implies that
\begin{align}
    &2\sum_{t=r}^s(\bx_{t, I}^\eta -\bu )^\top \nabla \ell_t^\eta(\bx_{t, I}^\eta) -\lambda\sum_{t=r}^s \|\bx_{t, I}^\eta - \bu\|^2\nonumber\\
    &\le \sum_{t=r}^s\frac{\|\bx_{t, I}^\eta -\bu\|^2-\|\bx_{t+1, I}^{\eta}-\bu\|^2
        }{\mu_t}+\sum_{t=r}^s\mu_t\|\nabla \ell_{t}^\eta (\bx_{t, I}^\eta )\|^2-\lambda\sum_{t=r}^s \|\bx_{t, I}^\eta - \bu\|^2\nonumber\\
        &=\sum_{t=r+1}^{s} \|\bx_{t,I}^\eta-\bu\|^2\underbrace{\left(\frac{1}{\mu_t}-\frac{1}{\mu_{t-1}}-\lambda \right)}_{=0} + \|\bx_{r,I}^\eta -\bu\|^2\underbrace{\left(\frac{1}{\mu_r} -\lambda \right)}_{=0}  -\underbrace{ \frac{\|\bx_{s+1,I}^\eta-\bu\|^2}{\mu_s}}_{\ge 0}\;  \nonumber\\
        & \quad\; +  \sum_{t=r}^s\mu_t\|\nabla\ell_t^\eta (\bx_{t, I}^\eta)\|^2
        \le \sum_{t=r}^s\mu_t\|\nabla\ell_t^\eta (\bx_{t, I}^\eta)\|^2\,.\label{eq:surrogate_SC_upper_bound}
    \end{align}
   Noticing that with any given $\bx_t \in \bbR^d$, $\ell_t^\eta$ 
   is $\lambda$-strongly-convex,  we apply \eqref{eq:surrogate_SC_upper_bound} to obtain that for all $\bu\in \tilde{\Theta}$, 
    \begin{align}
&2\sum_{t=r}^s\Exp\Big[\ell_t^\eta(\bx_{t, I}^\eta) -\ell_t^\eta(\bu) \mid \bx_{t}, \bx_{t, I}^\eta, \Lambda_T      \Big]\nonumber\\
& \le \Exp\Big[ 2\sum_{t=r}^s (\bx_{t, I}^\eta - \bu)^\top \nabla \ell_t^\eta(\bx_{t, I}^\eta)-\lambda \sum_{t=r}^s\|\bx_{t, I}^\eta -\bu\|^2\,\Big|\, \{\bx_{t}, \bx_{t, I}^\eta\}_{t=r}^s, \Lambda_T\Big] \nonumber\\
& \le \sum_{t=r}^s\mu_t\Exp\Big[\|\nabla\ell_t^\eta (\bx_{t, I}^\eta)\|^2\mid \bx_{t}, \bx_{t, I}^\eta, \Lambda_T\Big] \nonumber\\
&\stackrel{\text{(i)}}{\le} (G')^2\sum_{t=r}^s \mu_t \stackrel{\text{(ii)}}{\le} 
        \frac{(G')^2}{\lambda}\left( 1+\log (s-r+1)\right)\stackrel{\text{(iii)}}{\le}1+\log(s-r+1)\,,
    \end{align}
where (i)  is a result of \eqref{eq:UB_norm_grad_ell}, (ii) uses the bound $\sum_{k=1}^n\frac{1}{k}\le 1+ \log n$ for any $n\in \bbN^+$, 
 and (iii) uses the fact that given  $\eta\le \frac{1}{5GD}$ it holds that  
\begin{align}
    (G')^2= \left(\eta G +2\eta^2G^2 D\right)^2 = \eta^2G^2 + 4\eta^3 G^3D + 4\eta^4G^4D^2 \le \eta^2G^2 + \tfrac{4}{5}\eta^2  G^2  + \tfrac{4}{25}\eta^2  G^2\le \lambda.\nonumber
    \end{align}    
\end{proof}

\begin{lemma}[Meta-regret]\label{lem:meta_regret}
    For any GC interval $I=[r,s]\in \mc{I}$ and learning rate $\eta\in \mc{S}(s-r+1)$,  Algorithm~\ref{alg:metagrad_lean} satisfies
    \begin{align}
       \sum_{t=r}^s  \Exp\left[\ell_t^\eta(\bx_t)
        - \ell_t^\eta(\bx_{t,I}^\eta)\mid \bx_t,\bx_{t,I}^{\eta},\Lambda_T\right] = -\sum_{t=r}^s \Exp\left[\ell_t^\eta(\bx_{t,I}^\eta)\mid \bx_t,\bx_{t,I}^{\eta},\Lambda_T\right]\le 2\log(2s).
    \end{align}
\end{lemma}

\begin{proof}[Proof of Lemma \ref{lem:meta_regret}]The proof is similar to that of \cite[Lemma 6]{zhang2021dualadaptivity}.
By Jensen's inequality and the convexity of norms, we have for all $\bx\in \Theta,$
    \begin{align}
       | \left\langle \Exp[\bg_t|\bx_t, \Lambda_T], \bx_t- \bx\right\rangle| &\le \|\Exp[\bg_t|\bx_t, \Lambda_T]\| \|\bx_t-\bx\|\nonumber\\
       &\le \Exp\left[\|\bg_t\|\mid  \bx_t, \Lambda_T\right] \|\bx_t-\bx\| 
       \le GD\,,\label{eq:bound_gt_meta_regret}
    \end{align}
    which, given $\eta\le \frac{1}{5GD}$, implies that 
    \begin{align}
        \eta \left\langle \Exp[\bg_t|\bx_t, \Lambda_T], \bx_t- \bx\right\rangle \ge -\eta GD \ge -\tfrac{1}{5}\,.\label{eq:1_5_bound}
    \end{align}
    Using \eqref{eq:bound_gt_meta_regret}--\eqref{eq:1_5_bound} and applying the inequality $\ln(1+z) \ge z-z^2 $ for any $z\ge -\frac{2}{3}$ with $z=\eta\left\langle \Exp[\bg_t|\bx_t, \Lambda_T], \bx_t- \bx\right\rangle $, we obtain  for all $\bx\in\Theta$,
    \begin{align}
    \exp\left(-\Exp[\ell_t^\eta(\bx)\mid \bx_t, \Lambda_T]\right)
        &= \exp\left(\eta\left\langle \Exp[\bg_t|\bx_t, \Lambda_T], \bx_t-\bx\right\rangle - \eta^2G^2 \|\bx_t-\bx\|^2\right)\nonumber\\
        & \le  \exp\left(\eta\left\langle \Exp[\bg_t|\bx_t, \Lambda_T], \bx_t-\bx\right\rangle -\eta^2 \left\langle \Exp[\bg_t|\bx_t, \Lambda_T], \bx_t-\bx\right\rangle^2\right)\nonumber\\
        &\le 1+ \eta \left\langle \Exp[\bg_t|\bx_t, \Lambda_T], \bx_t- \bx\right\rangle.\label{eq:applied_ln(1+z)_lower_bound}
    \end{align}
Define shorthand  $\mathbb{F}_{t,I}^\eta = \{\bx_t,\bx_{t,I}^{\eta},\Lambda_T\}$, and $\mathbb{H}_{t,I}^\eta = \cup_{\tau\in[t]}\mathbb{F}_{\tau,I}^\eta$ for $t\in[T]$. Using \eqref{eq:applied_ln(1+z)_lower_bound},   we can write for every $t\in [T]$,
    \begin{align}
\sum_{E(I, \eta) \in \mc{A}_t}&\exp\left(-\Exp[ L_{t, I}^\eta\mid \mathbb{H}_{t,I}^\eta]\right) \nonumber\\
&=
         \sum_{E(I, \eta) \in \mc{A}_t}\exp\left(-\Exp [L_{t-1, I}^\eta \mid \mathbb{H}_{t-1,I}^\eta]\right)\exp\left(-\Exp[\ell_t^\eta (\bx_{t,I}^\eta)\mid \mathbb{F}_{t,I}^\eta]\right)\nonumber\\
         &
\leq  \sum_{E(I, \eta) \in \mc{A}_t}\exp\left(-\Exp [L_{t-1, I}^\eta \mid \mathbb{H}_{t-1,I}^\eta]\right)\left[1+ \eta \left\langle \Exp[\bg_t| \bx_t, \Lambda_T], \bx_t- \bx_{t,I}^\eta\right\rangle\right].\label{eq:Lt_less_than_Lt_1}
\end{align}
The second term on the RHS can be bounded as follows:
\begin{align}
     \sum_{E(I,\eta) \in \mc{A}_t}&\exp(-\Exp [L_{t-1, I}^\eta \nonumber\mid \mathbb{H}_{t-1, I}^\eta]) \left[\eta \left\langle \Exp[\bg_t|\bx_t, \Lambda_T], \bx_t- \bx_{t,I}^\eta\right\rangle\right] \\& =\Big\langle \Exp[\bg_t|\bx_t, \Lambda_T],\sum_{E(I,\eta) \in \mc{A}_t} \eta \exp(-\Exp [L_{t-1, I}^\eta \mid \mathbb{H}_{t-1,I}^\eta])(\bx_t-\bx_{t, I}^\eta) \Big\rangle \nonumber\\
&\stackrel{\text{(i)}}{\leq}
      \Big\langle \Exp[\bg_t|\bx_t, \Lambda_T], \sum_{E(I,\eta) \in \mc{A}_t}\eta\Exp \left[\exp(-L_{t-1, I}^\eta) \mid \mathbb{H}_{t-1,I}^\eta\right](\bx_t-\bx_{t, I}^\eta) \Big\rangle \nonumber\\
         &=
    \Big\langle \Exp[\bg_t|\bx_t, \Lambda_T], \Exp \Big[\underbrace{\sum_{E(I,\eta) \in \mc{A}_t}\eta\exp(-L_{t-1, I}^\eta) (\bx_t-\bx_{t, I}^\eta)}_{=0}\Big| \mathbb{H}_{t-1,I}^\eta\Big] \Big\rangle \stackrel{\text{(ii)}}{=}0\,, \label{eq:second_term_rhs}
    \end{align}
    where (i) applies  Jensen's inequality, and (ii) is due to the update rule of $\bx_t$ in \eqref{eq:TEWA_update_rule}. 
    Combining 
\eqref{eq:Lt_less_than_Lt_1}--\eqref{eq:second_term_rhs}  yields
    \begin{align}
         \sum_{E(I,\eta)\in \mc{A}_t}\exp\left(-\Exp [L_{t, I}^\eta \mid \mathbb{H}_{t,I}^\eta]\right) 
         \le   \sum_{E(I,\eta)\in \mc{A}_t} \exp\left(-\Exp [L_{t-1, I}^\eta \mid \mathbb{H}_{t-1,I}^\eta]\right) .\label{eq:Lt_less_than_Lt_1_simplified}
    \end{align}
   By summing both sides of \eqref{eq:Lt_less_than_Lt_1_simplified} over $t=1, \dots, s$ and rewriting, we obtain
    \begin{align}
      &   \sum_{E(I,\eta)\in \mc{A}_s}\exp\left(-\Exp [L_{s, I}^\eta \mid \mathbb{H}_{s,I}^\eta]\right) +\sum_{t=1}^{s-1}  \sum_{E(I,\eta)\in \mc{A}_t\setminus \mc{A}_{t+1}}\exp\left(-\Exp [L_{t, I}^\eta \mid \mathbb{H}_{t,I}^\eta]\right) +\nonumber\\
      & \;\;  \sum_{t=1}^{s-1}\sum_{E(I,\eta)\in \mc{A}_t \cap \mc{A}_{t+1} }\exp\left(-\Exp [L_{t, I}^\eta \mid \mathbb{H}_{t,I}^\eta]\right) \nonumber\\
      &\le  \sum_{E(I,\eta)\in \mc{A}_1}\exp\Big(-\Exp [ L_{0, I}^\eta ]\Big)  + \sum_{t=2}^s\sum_{E(I,\eta)\in \mc{A}_t\setminus \mc{A}_{t-1}} \exp\left(-\Exp [L_{t-1, I}^\eta \mid \mathbb{H}_{t-1,I}^\eta]\right)+\nonumber\\
      &\quad\; \sum_{t=2}^s
      \sum_{E(I,\eta)\in \mc{A}_t\cap \mc{A}_{t-1}} \exp\left(-\Exp [L_{t-1, I}^\eta \mid \mathbb{H}_{t-1,I}^\eta]\right) .\label{eq:sum_Lt_vs_Lt_1}
    \end{align}
Canceling  the equivalent last terms on both sides of \eqref{eq:sum_Lt_vs_Lt_1} and noting that 
 $L_{\tau, I}^\eta =0$ for $\tau= \min\{t: t\in I\}-1$  by construction (see line \ref{line:init_tewa} of Algorithm \ref{alg:metagrad_lean}), we obtain  for $s\ge 1$,
\begin{align}
     &   \sum_{E(I,\eta)\in \mc{A}_s}\exp\left(-\Exp [L_{s, I}^\eta \mid \mathbb{H}_{s,I}^\eta]\right)  + \sum_{t=1}^{s-1} \bigg[\sum_{E(I,\eta)\in \mc{A}_t\setminus \mc{A}_{t+1}}\exp\left(-\Exp [L_{t, I}^\eta \mid\mathbb{H}_{t,I}^\eta]\right) \bigg]\nonumber\\
      &\le  \sum_{E(I,\eta)\in \mc{A}_1}\exp\Big(\Exp[\underbrace{L_{0, I}^\eta}_{=0}]\Big)  + \sum_{t=2}^s\bigg[ \sum_{E(I,\eta)\in \mc{A}_t\setminus \mc{A}_{t-1}} \exp\Big(-\Exp [\underbrace{L_{t-1, I}^\eta}_{=0} \mid \mathbb{H}_{t-1,I}^\eta]\Big) \bigg] \nonumber\\
      &= \sum_{E(I,\eta)\in \mc{A}_1}\exp(0) + \sum_{t=2}^s\sum_{E(I,\eta)\in \mc{A}_t\setminus \mc{A}_{t-1}}\exp(0)\nonumber\\
      & = |\mc{A}_1| + \sum_{t=2}^s|\mc{A}_t\setminus \mc{A}_{t-1}|\le  \sum_{t=1}^s\big|\mc{A}_t\big| \nonumber\\
      &\stackrel{\text{(i)}}{\le} \sum_{t=1}^s (1+ \floor{\log_2 t})\cdot \left(1+\ceil{\tfrac{1}{2} \log_2 t}\right) \le \sum_{t=1}^s (1+  \log_2 t)^2
    \stackrel{\text{(ii)}}{\le}   4s^2 \label{eq:inequality_4s^2}\,,
\end{align}
where (i) applies \eqref{eq:nt_bound}, and (ii) is due to $1+\log_2 t\le 2\sqrt{t}\; \forall t\ge1$. 
 Since $\exp(x)> 0$ for $x\in \bbR$, Eq.~\eqref{eq:inequality_4s^2} implies that for  any GC interval $I=[r,s]\in\mc{I}$ and learning rate $\eta\in \mc{S}(|I|)$, 
\begin{align}
&\exp\left(-\Exp [L_{s, I}^\eta \mid \mathbb{H}_{s,I}^\eta]\right)=
     \exp\left(-\sum_{t=r}^s\Exp [\ell_t^\eta(\bx_{t, I}^\eta) \mid \mathbb{F}_{t, I}^\eta]\right)\le 4s^2 .
     \end{align}
   Taking the logarithm of both sides completes the proof.
\end{proof}

\subsection{Proof of Corollary \ref{cor:TEWA_R_swi_known}}\label{sec:appx_TEWA_cor}
We first restate Corollary \ref{cor:TEWA_R_swi_known} and then provide the proof. Recall that for clarity we drop the $\ceil{\cdot}$ operators from the expressions for $\sfB$ and assume without loss of generality the expressions take integer values.
\corTEWAswi*
    \begin{proof}[Proof of Corollary \ref{cor:TEWA_R_swi_known}]
 We begin by applying the first result in Proposition~\ref{prop:conversions} with the adaptive regret guarantees in Theorem~\ref{thm:interval_regret} to obtain  switching regret guarantees. For known $S$, Algorithm~\ref{alg:metagrad_lean} with parameter $\sfB= \frac{T}{S}$ achieves in the general convex case,
 \begin{align}\label{eq:R_swi_TEWA}
        R^{\swi} (T, S) \leq  2C\sqrt{d}S^{\frac{1}{4}}T^{\frac{3}{4}} + 2(C_1  +  \tfrac{C_2}{d} +\tfrac{4}{d^2})d^2S\,,
    \end{align}
    and in the case where $f_t$ is $\alpha$-strongly-convex and $\argmin_{\bx\in \bbR^d} f_t(\bx)\in \Theta$   for all $t\in [T]$, 
\begin{align}\label{eq:R_swi_TEWA_sc}
        R^{\swi} (T, S) \leq  2C' d\sqrt{ST} + 2(C_1'  + \tfrac{C_2'}{d} + \tfrac{4}{d^2})d^2S\,,
    \end{align}
where $C, C_1, C_2, C', C_1', C_2'>0$ are the  terms defined in \eqref{eq:C_def}--\eqref{eq:C2_def} and \eqref{eq:C_prime_def}--\eqref{eq:C2_prime_def} which are polylogarithmic in $T$ and $\sfB$.  
  When  $S$ and $\Delta$ are both known, we use  \eqref{eq:R_swi_TEWA}--\eqref{eq:R_swi_TEWA_sc} and apply  the second result in Proposition~\ref{prop:conversions} to bound $R^{\dyn}(T, \Delta, S)$.  Specifically, for general convex losses,   Algorithm~\ref{alg:metagrad_lean} with  
    $\sfB= \frac{T}{S} \vee \big(\frac{\sqrt{d}T}{\Delta}\big)^{\frac{4}{5}} $  yields
    \begin{align*}
         R^\dyn(T, \Delta, S)\le  R^\swi(T, S)\wedge F^\dyn(T, \Delta) \,, 
    \end{align*}
    where  $F^\dyn(T, \Delta)\coloneqq (2C+1) d^{\frac{2}{5}} \Delta^{\frac{1}{5}} T^{\frac{4}{5}} + 2(C_1+\tfrac{C_2}{d}+\tfrac{4}{d^{2}})d^{\frac{8}{5}}\Delta^{\frac{4}{5}}T^{\frac{1}{5}}$. 
    For strongly-convex losses with minimizers  inside $\Theta$,  Algorithm~\ref{alg:metagrad_lean} 
 with  $\sfB=  \frac{T}{S}\vee \big(\frac{dT}{\Delta}\big)^{\frac{2}{3}} $ gives 
    \begin{align*}
    R^\dyn(T, \Delta, S)\le  R^\swi(T, S) \wedge F^\dyn_\scvx(T, \Delta)\,.
\end{align*}
where $ F^\dyn_\scvx(T, \Delta)\coloneqq (2C'+1)   d^{\frac{2}{3}} \Delta^{\frac{1}{3}} T^{\frac{2}{3}} + 2(C_1'+\tfrac{C_2'}{d}+ \tfrac{4}{d^{2}})  d^{\frac{4}{3}} \Delta^{\frac{2}{3}} T^{\frac{1}{3}}$. 
    Finally, for  known $P$ we use  \eqref{eq:R_swi_TEWA}--\eqref{eq:R_swi_TEWA_sc} and apply the third result in Proposition~\ref{prop:conversions} to bound $R^\pat(T, P)$.
    For the general convex case, taking $\sfB=\big(\frac{r\sqrt{d}T}{P}\big)^{\frac{4}{5}}$  gives 
\begin{align*}
    R^\pat(T, P)\le (2C+1) r^{-\frac{1}{5}} d^{\frac{2}{5}}P^{\frac{1}{5}}T^{\frac{4}{5}}  
     +2(C_1+ \tfrac{C_2}{d}+ \tfrac{4}{d^{2}})r^{-\frac{4}{5}}  d^{\frac{8}{5}} P^{\frac{4}{5}}T^{\frac{1}{5}}\,.
\end{align*}
For strongly-convex losses with minimizers  inside $\Theta$,  taking  
$\sfB=\big(\frac{rdT}{P}\big)^{\frac{2}{3}}$ yields 
\begin{align*}
    R^\pat(T, P)\le (2C'+1) r^{-\frac{1}{3}} d^{\frac{2}{3}} P^{\frac{1}{3}}T^{\frac{2}{3}} + 2(C_1'+\tfrac{C_2'}{d} + \tfrac{4}{d^2}) r^{-\frac{2}{3}} d^{\frac{4}{3}}P^{\frac{2}{3}}T^{\frac{1}{3}}.
\end{align*}
    \end{proof}

%% file: appx_upper_bound_unknown_S.tex

\subsection{Parameter-free upper bounds}\label{appx:TEWA_BoB}
Corollary \ref{cor:TEWA_R_swi_known} presents the optimal choice of parameter $\sfB$ for TEWA-SE when  $S$, $\Delta$ and $P$  are  known.  
 When the non-stationarity measures are unknown, the optimal $\sfB$ cannot be directly computed, and we therefore employ   the Bandit-over-Bandit (BoB) framework from \cite{cheung19bob} to adaptively select   $\sfB$ from a  prespecified set $\mc{B}=\{2^i: i=0, 1, \dots, \floor{\log_2 T}\}$. BoB has been used in \cite{WeiL21}
 in a similar fashion to obtain parameter-free algorithms. Specifically,  BoB divides  the time horizon   into  $E=\ceil{T/ L}$  epochs each with length $L$, denoted by $(I_e)_{e=1}^{E}$ (where the last epoch may be shorter than $L$). In the first epoch, it runs TEWA-SE with $\sfB=\sfB_1$ which is randomly selected from $\mc{B}$. For subsequent epochs, it uses the cumulative empirical loss on the current epoch $ e-1$ to select  $\sfB_{e}\in \mc{B}$ for the next epoch via  EXP3 \cite{auer2002nonstochastic}. That is,  
 BoB computes 
\begin{align}
    p_{e, i}=(1-\gamma) \frac{s_{e, i}}{\sum_{i'\in [|\mc{B}|]} s_{e, i'}} + \frac{\gamma}{|\mc{B}|} \quad\forall i\in [|\mc{B}|], \quad \text{with}\quad \gamma=1 \wedge \sqrt{\frac{|\mc{\sfB}| \ln (|\mc{B}|)}{(\textrm{e}-1) E}} \,,\label{eq:BoB_p_gamma_def}
\end{align}
where $\textrm{e}$ denotes the base of the exponential function, and  then samples $i_e=i$ with probability $p_{e, i}$ yielding  $ \sfB_{e}=2^{i_e-1}$.\footnote{We adopt  clipping (by $\gamma$) following \cite{auer2002nonstochastic, cheung19bob}, though $\gamma=0$ suffices as discussed in \cite{stoltz2005thesis} and \cite[Section 11.6]{lattimore2020bandit}.}  
For $i\in [|\mc{B}|]$, initialized with  $ s_{0, i} =1$, the quantity $s_{e,i}$ for $e\in \bbN^+$ is updated by
computing 
\begin{align}\label{eq:update_s}
 s_{e+1, i}= 
        s_{e, i} \exp\Big(\frac{\gamma}{|\mc{B}|} \widehat{r}_{e, i} \Big)\,,
\end{align}
where with $M_T=1+2\sigma \sqrt{\log(T+1)}$, the importance-weighted reward $\widehat{r}_{e, i} $ takes the form 
\begin{align}\label{eq:reward_def}
\widehat{r}_{e, i}=
    \begin{cases}
   \left(\frac{1}{2} +  \frac{1}{2LM_T}\sum_{t\in I_e} (1-y_t)\right)  / p_{e, i} &\text{if}\; i=i_e  \\
        0& \text{otherwise}\,.
    \end{cases}
\end{align}
Note that conditioned on the event $\Lambda_T = \big\lbrace|\xi_t| \le 2\sigma \sqrt{\log (T+1)}, \; \forall t\in [T]\big\rbrace$  defined above  \eqref{eq:def_hat_G},  the absolute total reward in each epoch is bounded by $Q\coloneqq\max_{e\in [E]}\big|\sum_{t\in I_e} (1-y_t)\big| \le L M_T$, which ensures the rescaled reward $ \frac{1}{2} +  \frac{1}{2LM_T}\sum_{t\in I_e} (1-y_t) $ in \eqref{eq:reward_def} remains bounded within $ [0,1]$. The pseudo-code for TEWA-SE equipped with BoB is provided in Algorithm \ref{alg:BoB}, with theoretical guarantees detailed in Corollary \ref{cor:BoB}.


\begin{algorithm}
\caption{TEWA equipped with Bandit-over-Bandit (BoB)} \label{alg:BoB}
\textbf{Input:} 
$d, T, L, E=\ceil{T/L}, (I_e)_{e=1}^E, \mc{B}=\{2^i: i=0, 1, \dots, \floor{\log_2 T}\}$, and $\gamma\in (0,1)$ as defined in \eqref{eq:BoB_p_gamma_def}
\\
\textbf{Initialize:} $ s_{0, i} =1\; \forall i\in [|\mc{B}|]$
\begin{algorithmic}[1]
\For{$e=1, 2, \dots, E$}
\State Compute $p_{e, i}$ according to \eqref{eq:BoB_p_gamma_def} $\forall i\in [|\mc{B}|]$
\State  Sample $i_e=i $ with  probability $p_{e, i}$\,, and select  $\sfB_e=2^{i_e-1}\in \mc{B}$
\For{$t\in I_e$}
\State Run TEWA-SE with $\sfB=\sfB_e$ to select action $\bz_t$ and observe losses $ y_t=f_t(\bz_t)+ \xi_t$
\EndFor
\State Update $s_{e+1, i}$ according to \eqref{eq:update_s} $\forall i\in [|\mc{B}|]$
\EndFor
    \end{algorithmic}
\end{algorithm}



\begin{corollary}[TEWA-SE with BoB]\label{cor:BoB}
  Consider any horizon $T\in \bbN^+$  and assume that, for all $t \in [T]$, the loss $f_t$ is convex, or strongly-convex with $\argmin_{\bx\in\bbR^d}f_t(\bx)\in \Theta$ (referred to as the strongly-convex (SC) case).
  
  Then, for the general convex case, Algorithm~\ref{alg:BoB}   with epoch size $L=(dT)^{\frac{2}{3}}$ attains all the regret bounds from Corollary \ref{cor:TEWA_R_swi_known} plus an additional  term of 
  $d^{\frac{1}{3}}T^{\frac{5}{6}}  + d^{\frac{4}{3}}T^{\frac{1}{3}}$. For the SC case, Algorithm~\ref{alg:BoB}   with epoch size $L=d\sqrt{T}$ satisfies all the regret bounds from Corollary \ref{cor:TEWA_R_swi_known} plus an additional term of $d^{\frac{1}{2}} T^{\frac{3}{4}} + d\sqrt{T}$.  Both results omitted polylogarithmic factors.
\end{corollary}

\begin{proof}[Proof of Corollary \ref{cor:BoB}]For brevity, we suppress terms that are polylogarithmic in $T$ using $\lesssim $ in this proof. 
        For all $\sfB^\dagger\in \mc{B}$,  we have
\begin{align}
     R^\dyn (T) &= \sum_{t=1}^T\Exp\left[ f_t(\bz_t) - \min_{\bz\in \Theta}f_t(\bz)\right] \nonumber\\
    &= 
    \sum_{t=1}^T\Exp\Big[ f_t(\bz_t) - f_t\left(\bz_t(\sfB^\dagger)\right)\Big] + 
    \sum_{t=1}^T\Exp\left[f_t\left(\bz_t(\sfB^\dagger)\right) - \min_{\bz\in\com}f_t(\bz)\right]\nonumber\\
    &= 
\underbrace{\sum_{e=1}^{E}\sum_{t\in I_e}\Exp\Big[ f_t\left(\bz_t(\sfB_{e})\right) - f_t\left(\bz_t(\sfB^\dagger)\right)\Big]}_{\text{term I}} + \underbrace{
    \sum_{e=1}^{E}\sum_{t\in I_e}\Exp\left[f_t\left(\bz_t(\sfB^\dagger)\right) - \min_{\bz\in\com}f_t(\bz)\right]}_{\text{term II}}\,,\label{eq:decompse_R_BoB}
\end{align}
where $\bz_t(\sfB_e)$ represents the  actual action taken by TEWA-SE in round $t$ of epoch $e$, and   $\bz_t(\sfB^\dagger)$ denotes the hypothetical action that TEWA-SE would have chosen had its $\sfB$ parameter been set to $\sfB^\dagger$. 
Term I in \eqref{eq:decompse_R_BoB} can be bounded by applying the classical analysis of EXP3 from \cite[Corollary 3.2]{auer2002nonstochastic}, combined  with \eqref{eq:bound_P_Lambda_T_c}, as follows
\begin{align}
   \forall \sfB^\dagger \in \mc{B}:\quad  \text{term I}&\le 4\sqrt{\textrm{e}-1}\sqrt{E|\mc{B}|\log |\mc{B}|}\cdot \Exp\left[Q\right] \nonumber\\
    &\lesssim   \sqrt{E }\cdot (\Prob(\Lambda_T)\Exp\left[Q \mid \Lambda_T\right] + \Prob(\Lambda_T^c)\Exp\left[Q \mid \Lambda_T^c\right])\nonumber\\
    &\lesssim  \sqrt{T/L } \cdot (LM_T + \tfrac{2}{T}\cdot L)
    \lesssim \sqrt{T L}\,.\label{eq:term1_BoB}
\end{align}
To bound term II,  we introduce shorthand $F^\ada(\sfB, T)$ to refer to the upper bound on $R^\ada(\sfB, T)$ in Theorem~\ref{thm:interval_regret}, and  $F^\swi(T, S), F^{\dyn}(T, \Delta, S)$ and $F^\pat(T, P)$ to refer to the upper bounds on $R^\swi(T, S), R^{\dyn}(T, \Delta, S)$ and $R^\pat(T, P)$  in Corollary~\ref{cor:TEWA_R_swi_known} for known $S, \Delta$ and $P$. We  also use   $S_e=1+ \sum_{t\in I_e }\ind{f_t\neq f_{t-1}}$. 
By choosing $ \sfB^\dagger = 2^{i^\dagger}$  in the analysis with $ i^\dagger=    \floor{\log_2\tfrac{T}{S}} \wedge \floor{\log_2 L} $, term II can be bounded in terms of the number of switches $S$ by 
\begin{align}
    \text{term II}&\le  \sum_{e=1}^E   \left(\ceil{\tfrac{L}{ \sfB^\dagger} }+ S_e\right)R^{\ada}(\sfB^\dagger, L)
     \le \big(\tfrac{T}{\sfB^\dagger} + S+ E\big)R^{\ada}(\sfB^\dagger, L)\nonumber\\
     &\le F^\swi(T, S)+\tfrac{T}{L}  F^\ada(L, L)  
    \,. \label{eq:term2_BoB}
\end{align}
Combining \eqref{eq:term1_BoB} and \eqref{eq:term2_BoB}, we obtain
\begin{align}
    R^\swi(T, S)&\lesssim F^\swi(T, S)  + \left[\tfrac{T}{L}F^\ada(L, L)+ \sqrt{TL}\right]\nonumber\\
    &\lesssim F^\swi(T, S)  + \begin{cases}
     d^{\frac{1}{3}}T^{\frac{5}{6}} + d^{\frac{4}{3}}T^{\frac{1}{3}}\\
       d^{\frac{1}{2}} T^{\frac{3}{4}} + d\sqrt{T} &(\text{SC})\,,
    \end{cases}
\end{align}
where we used $L=(dT)^{\frac{2}{3}}$ for the general convex case, and  $L=d\sqrt{T}$ for the strongly-convex case. 
Following  similar steps, by  choosing $\sfB^\dagger = 2^{i^\dagger}$ in the analysis with $ i^\dagger=   (\floor{\log_2\tfrac{T}{S}}\vee \floor{ \log_2(\sfB_\Delta)})\wedge \floor{\log_2 L} $ where $ \sfB_\Delta=\big(\frac{\sqrt{d}T}{\Delta}\big)^{\frac{4}{5}}$ for the general convex case or $\sfB_\Delta=\big(\frac{dT}{\Delta}\big)^{\frac{2}{3}}$ for the strongly-convex case, 
we obtain
\begin{align}
    R^\dyn(T, \Delta, S)&\lesssim F^\dyn(T,\Delta, S)   + \begin{cases} 
    d^{\frac{1}{3}}T^{\frac{5}{6}} +  d^{\frac{4}{3}}T^{\frac{1}{3}} \\
       d^{\frac{1}{2}} T^{\frac{3}{4}} + d\sqrt{T} &(\text{SC})\,.
    \end{cases}
\end{align}
The bound on $R^\pat(T, P)$ can be established analogously.  
\end{proof}

%% file: appx_lower_bound.tex
\section{Proofs of lower bounds}
\label{sec:lowerbounds_proof}
We call $\pi = \{\bz_t\}_{t=1}^{\infty}$ a randomized procedure if $\bz_t =\Phi_t(\{\bz_k\}_{k=1}^{t-1},\{y_k\}_{k=1}^{t-1})$ where $\Phi_t$ are Borel functions, and $\bz_1\in\mathbb{R}^d$ is deterministic. 
We emphasize that, throughout this section, we assume the noise variables $\{\xi_t\}_{t=1}^{T}$ are independent with cumulative distribution function $F$ satisfying the condition
\begin{align}\label{con:lw}
    \int\log\left(\d F\left(u\right)/\d F\left(u+v\right)\right) \d F(u)\leq I_0v^2,\quad\quad |v|<v_0,
\end{align}
for some $0 < I_0 < \infty, 0 < v_0 \leq \infty.$ This condition holds, for instance, if $F$ has a sufficiently smooth density with finite Fisher information. In the special case where $F$ is Gaussian, the inequality \eqref{con:lw} holds with $v_0 = \infty$. Note that Gaussian noise also satisfies our sub-Gaussian noise assumption   in  Section \ref{sec:intro}, 
which is used in the proof of the upper bounds.

We first restate and prove Theorem \ref{thm:lowerbound}, which establishes a lower bound on $R^\dyn(T, \Delta, S)$, and then  present and prove Theorem \ref{thm:2lowerbound}, which establishes a  lower bound on $R^\pat(T, P)$.
\theoremLowerBound*
\begin{proof}[Proof of Theorem \ref{thm:lowerbound}]
   Let $\eta_0:\mathbb{R}\to\mathbb{R}$ be an infinitely many times differentiable function that satisfies
\begin{align*}
    \eta_0(x)  \begin{cases}
    =1 \quad\quad &\text{if }|x|\leq 1/4\,,\\
    \in (0,1)\quad\quad &\text{if }1/4<|x|< 1\,,\\
    =0 \quad\quad &\text{if }|x|\geq 1\,.
    \end{cases}
\end{align*}
Denote by $\Omega =\{-1,1\}^{\dx}$   the set of binary sequences of length $d$, and let $\eta(x) = \int_{-\infty}^{x}\eta_0(u)\d u$. Consider the set of functions $f_{\bomega}:\mathbb{R}^{\dx}\to\mathbb{R}$ with $\bomega=(\omega_1, \dots, \omega_d) \in \{-1,1\}^{\dx}$ such that:
\begin{equation}\label{lw:function}
\begin{aligned}
    f_{\bomega}(\bx) = \alpha\norm{\bx}^2+\lbfuncconst h^{2}\left(\sum_{i=1}^{\dx}\omega_i\eta\left(\frac{x_i}{h}\right)\right),\quad\quad \bx = (x_1,\dots,x_{\dx}),
\end{aligned}
\end{equation}
where 
$h = \min\left(\dx^{-\frac{1}{2}},\left(\frac{T}{S}\right)^{-\frac{1}{4}},(\frac{dT}{\Delta})^{-\frac{1}{6}}\right)$, and $\lbfuncconst>0$ is to be assigned later. Let $L' = \max_{x\in\mathbb{R}}|\eta''\left(x\right)|$.
By \cite[Lemma 10]{akhavan2024contextual} we have that if $\lbfuncconst\leq \min\left(1/2\eta(1), \alpha/L'\right)$ then $f_{\bomega}\in \mathcal{F}_{\alpha}$. Moreover, if $\lbfuncconst\leq \alpha/2$, the equation $\nabla f_{\bomega}(\bx) = 0$ has the solution
\begin{align}\label{eq:def_x_star}
    \bx^{*}(\bomega) = \left(x^{*}_{1}(\bomega),\dots, x^{*}_{\dx}(\bomega)\right),\quad \text{with}\quad x^{*}_{i}(\bomega) = -\frac{h\lbfuncconst\omega_i}{2\alpha} \quad   \text{for}\;1\leq i\leq \dx\,.
\end{align}
This  is the unique minimizer of  $ f_{\bomega}$ and belongs to  $\bx^{*}(\bomega)\in \com=\mathbb{B}^d$ because  
\begin{align*}
    \norm{\bx^{*}(\bomega)}^2\leq \frac{ h^2 \lbfuncconst^2  \dx}{4\alpha^2}\leq \frac{1}{16}.
\end{align*}

We consider the following adversarial protocol. At the beginning of the game, the adversary selects $N_c = \min(S , (T\Delta^2/d^2)^{\frac{1}{3}})$ points from $\Omega$, sampled uniformly at random with replacement. Here without loss of generality we assumed that $(T\Delta^2)^{\frac{1}{3}}$ is a positive integer. Denote these points by $\{\bomega_k\}_{k=1}^{N_c}$, and then for each $k = 1,2,\dots,N_c$, let
\begin{align*}
    f_{(k-1)T/N_c +1}=\dots=f_{kT/N_c}=f_{\bomega_k}\,.
\end{align*}
For any $\bomega,\bomega'\in \Omega$ let $\rho(\bomega,\bomega')= \sum_{i=1}^{\dx}\mathbb{1}\left(\omega_i\neq\omega_i'\right)$ be the Hamming distance between $\bomega$ and $\bomega'$, with $\bomega=\left(\omega_{1},\dots,\omega_{d}\right)$ and $\bomega'=\left(\omega'_{1},\dots,\omega'_{d}\right)$. By construction, 
 $N_c\le S$ and 
\begin{align*}
\sum_{k=2}^{N_c}\max_{\bx\in\com}|f_{\bomega_{k-1}}(\bx)-f_{\bomega_{k}}(\bx)|\leq 2 \lbfuncconst h^2\eta(1)\sum_{k=2}^{N_c}\rho(\bomega_{k-1},\bomega_{k} )\leq \Delta\,.
\end{align*}

For any fixed $\bomega_1,\dots,\bomega_{N_c}\in\Omega$, and $1\leq t\leq T$, denote $\Gamma = [\bomega_1\,|\,\dots\,|\,\bomega_{N_c}]$ as the matrix whose columns are the $\bomega_k$'s. Denote by $\mathbf{P}_{\Gamma, t}$ the probability measure corresponding to the joint distribution of $\{\bz_k, y_k\}_{k=1}^t$ where $y_k = f_{k}(\bz_k)+\xi_k$ with independent identically distributed $\xi_k$’s such that $\eqref{con:lw}$ holds and $\bz_k$'s are chosen by the algorithm $\pi$. We have 
\begin{align}
\d\mathbf{P}_{\Gamma,t}\left(\bz_{1:t}, y_{1:t}\right) &= \d F\left(y_1-f_{1}\left(\bz_{1}\right)\right)\prod_{\tau=2}^{t}\d F\left(y_\tau-f_{\tau}\left(\Phi_\tau\left(\bz_{1:\tau-1}, y_{1:\tau-1}\right)\right)\right)
   \nonumber \\&= \d F\left(y_1-f_{\bomega_1}\left(\bz_1\right)\right)\prod_{\tau=2}^{t}\d F\left(y_\tau-f_{\bomega_{k_\tau}}\left(\Phi_\tau\left(\bz_{1:\tau-1}, y_{1:\tau-1}\right)\right)\right),\label{lw:dist}
\end{align}
where $k_\tau = \floor{(\tau-1)N_c/T} + 1$.
(We omit explicit mention of the dependence of $\mathbf{P}_{\Gamma,t}$ and $\Phi_\tau$ on $\bz_2, \dots, \bz_{\tau-1}$, since $\bz_\tau$ for $ \tau \geq 2$ is a Borel function of $\bz_1,y_1,\dots,\bz_{\tau-1}, y_{\tau-1}$.) Let $\Exp_{\Gamma,t}$ denote the expectation w.r.t. $\mathbf{P}_{\Gamma,t}$.

Note that by $\alpha$-strong convexity of $f_{\bomega}$ and the fact that  $\bx^*(\bomega) \in \argmin_{\bx\in \bbR^d} f_{\bomega}(\bx) $ from \eqref{eq:def_x_star}, we have
\begin{equation}\label{eq:strongc}
\sum_{t=1}^{T}\Exp_{\Gamma,t}\left[f_{\bomega_{k_t}}(\bz_t) - \min_{\bx\in\com}f_{\bomega_{k_t}}(\bx)\right] \ge \frac{\alpha}{2} \sum_{t=1}^{T}\Exp_{\Gamma,t}\left[\norm{\bz_t-\bx^{*}(\bomega_{k_t})}^2\right]\,. 
\end{equation}
Define the nearest-neighbour estimator
\begin{align*}
    \hat{\bomega}_t\in \argmin_{\bomega\in \Omega}\norm{\bz_t-\bx^{*}(\bomega)}.
\end{align*}
Using this combined with the triangle inequality, we have   $\norm{\bx^{*}(\hat{\bomega}_t) -\bx^{*}(\bomega_{k_t})}\leq\norm{\bz_t - \bx^{*}(\hat{\bomega}_t)}+\norm{\bz_t-\bx^{*}(\bomega_{k_t})}\leq 2\norm{\bz_t-\bx^{*}(\bomega_{k_t})}$. Together with \eqref{eq:def_x_star} this implies that
\begin{align*}
    \Exp_{\Gamma,t}\left[\norm{\bz_t-\bx^{*}(\bomega_{k_t})}^{2}\right]&\geq \frac{1}{4}\Exp_{\Gamma,t}\left[\norm{\bx^{*}(\hat{\bomega}_t)-\bx^{*}(\bomega_{k_t})}^{2}\right]
    =\frac{h^{2}\lbfuncconst^2}{4 \alpha^{2}}\Exp_{\Gamma,t}\left[\rho(\hat{\bomega}_t,\bomega_{k_t})\right].
\end{align*}
Summing over $1,\dots, T$, then taking the maximum over $\Gamma = [\bomega_1|\dots|\bomega_{N_c}]$ and then the minimum over all estimators $\hat{\bomega}_{1}, \dots, \hat{\bomega}_{T}$ with values in $\Omega$, we get
\begin{align}
\max_{\Gamma\in\Omega^{N_c}}\sum_{t=1}^{T}\Exp_{\Gamma,t}\left[\norm{\bz_t-\bx^{*}(\bomega_{k_t})}^{2}\right] \geq 
\frac{h^{2}\lbfuncconst^2}{4 \alpha^{2}}\underbrace{\min_{\hat{\bomega}_1,\dots,\hat{\bomega}_T\in \Omega}\max_{\Gamma\in\Omega^{N_c}}\sum_{t=1}^{T}\sum_{i=1}^{\dx}\Exp_{\Gamma,t}\left[\mathbb{1}\left(\hat{\omega}_{t,i}\neq\omega_{k_t,i}\right)\right]}_{\text{term I}}.\label{eq:def_LB_termI}
\end{align}
For term I, lower bounding the maximum with the average we can write
\begin{align*}
    \text{term I} &\geq 2^{-\dx N_c}\min_{\hat{\bomega}_1,
\dots,\hat{\bomega}_T\in \Omega}\sum_{t=1}^{T}\sum_{\Gamma\in\Omega^{N_c}}\sum_{i=1}^{\dx}\Exp_{\Gamma,t}\left[\mathbb{1}\left(\hat{\omega}_{t,i}\neq\omega_{k_t,i}\right)\right]
\\&\geq 2^{-\dx N_c}\sum_{t=1}^{T}\sum_{\Gamma\in\Omega^{N_c}}\sum_{i=1}^{\dx}\min_{\hat{\omega}_{t,i}\in \{-1,1\}}\Exp_{\Gamma,t}\left[\mathbb{1}\left(\hat{\omega}_{t,i}\neq\omega_{k_t,i}\right)\right].
\end{align*}
Next, for each $i=1,\dots,\dx$, define $\Gamma_i^{k_t} =\{[\bomega_1|\dots|\bomega_{N_c}]: \bomega_1,\dots,\bomega_{N_c}\in\Omega, \omega_{k_t,i}=1\}$. Given any $\Gamma\in\Gamma_i^{k_t}$, let $\bar{\Gamma} = [\bar{\bomega}_1|\dots|\bar{\bomega}_{N_c}]$ such that $\bar{\omega}_{k,j}=\omega_{k,j}$ for any $k\neq k_t$, 
and let $\bar{\omega}_{k_t,i}=-1$ and $\bar{\omega}_{k_t, j}=\omega_{k_t, j}$ for $j\ne i$.  Hence,
\begin{align*}
    \text{term I} &\geq
     2^{-\dx {N_c}}\sum_{t=1}^{T}\sum_{\Gamma\in\Omega^{N_c}}\sum_{i=1}^{\dx}\min_{\hat{\omega}_{t,i}\in\{-1,1\}}\left(\Exp_{\Gamma,t}\left[\mathbb{1}\left(\hat{\omega}_{t,i}\neq1\right)\right] +\Exp_{\bar{\Gamma},t}\left[\mathbb{1}\left(\hat{\omega}_{t,i}\neq-1\right)\right]\right)\\
    &\geq 
    \frac{1}{2}\sum_{t=1}^{T} \sum_{i=1}^{\dx}\min_{\Gamma \in \Gamma_{i}^{k_t}}\min_{\hat{\omega}_{t,i}\in \{-1,1\}}\left(\Exp_{\Gamma,t}\left[\mathbb{1}\left(\hat{\omega}_{t,i}\neq1\right)\right] +\Exp_{\bar{\Gamma},t}\left[\mathbb{1}\left(\hat{\omega}_{t,i}\neq-1\right)\right]\right).
\end{align*}
Thus, we can write
\begin{align*}
\text{KL}\left(\mathbf{P}_{\Gamma,t}||\mathbf{P}_{\bar{\Gamma},t}\right) &= \int \log\left(\frac{\d\mathbf{P}_{\Gamma,t}}{\d\mathbf{P}_{\bar{\Gamma},t}}\right)\d\mathbf{P}_{\Gamma,t}\\
&=\int\bigg[\log\left(\frac{\d F(y_1-f_{\bomega_1}(\bz_1))}{\d F(y_1-f_{\bar{\bomega}_1}(\bz_1))}\right)+
\\&\phantom{0000000}+\sum_{\tau=2}^{t}\log\left(\frac{\d F(y_\tau-f_{\bomega_{k_\tau}}\left(\Phi_\tau\left(\bz_{1:\tau-1},y_{1:\tau-1}\right))\right)}{\d F(y_\tau-f_{\bar{\bomega}_{k_\tau}}\left(\Phi_\tau\left(\bz_{1:\tau-1},y_{1:\tau-1}\right))\right)}\right)\bigg]
\\&\phantom{000000000}\d F\left(y_1-f_{\bomega_1}\left(\bz_{1}
\right)\right)\prod_{\tau=2}^{t}\d F\left(y_\tau-f_{\bomega_{k_\tau}}\left(\Phi_\tau\left(\bz_{1:\tau-1},y_{1:\tau-1}\right)
\right)\right)
\\&\leq I_0\sum_{\tau=1}^{t}\max_{\bx\in\com
}|f_{\bomega_{k_\tau}}(\bx) - f_{\bar{\bomega}_{k_\tau}}(\bx)|^{2}
\leq 4TN_c^{-1}I_0\lbfuncconst^2h^{4}\eta^2(1).
\end{align*}
Since $h \leq \min(\left(\frac{S}{T}\right)^{\frac{1}{4}},(\frac{\Delta}{dT})^{\frac{1}{6}})$, and by choosing $\lbfuncconst\leq \sqrt{\log(2)/(4I_0\eta^{2}(1))}$, we have $\text{KL}(\mathbf{P}_{\Gamma,t}||\mathbf{P}_{\bar{\Gamma},t})\le \log(2)$. Hence, Theorem 2.12 of \cite{Tsybakov09} gives
\begin{align*}
    \text{term I}\geq \frac{T\dx}{4}\exp(-\log(2)) = \frac{T\dx}{8}.
\end{align*}
Substituting this into \eqref{eq:def_LB_termI} and our overall bound \eqref{eq:strongc} yields
\begin{align*}
\max_{\Gamma\in\Omega^{N_c}}\sum_{t=1}^{T}\Exp_{\Gamma,t}\left[f_{\bomega_{k_t}}(\bz_t) - \min_{\bx\in\com}f_{\bomega_{k_t}}(\bx)\right]\ge  \frac{\alpha}{2}\cdot \frac{h^{2}\lbfuncconst^2 }{4\alpha^{2}}\cdot \frac{Td}{8} = \frac{h^{2}\lbfuncconst^2 T\dx}{64\alpha }. 
\end{align*}
Finally, substituting the definition of $h$ and noting that $\lbfuncconst$ is independent of $d, T,S$ and $\Delta$ completes the proof.
\end{proof}

\begin{theorem}\label{thm:2lowerbound}
Let $\com = \mathbb{B}^d$. For $\alpha>0$ denote by $\mathcal{F}_\alpha$ the class of $\alpha$-strongly convex and smooth functions. Let $\pi=\{\bz_t\}_{t=1}^{T}$ be any randomized algorithm. Then there exists $T_0>0$ such that for all $T\geq T_0$ it holds that 
\begin{align*}
\sup_{f_1,\dots,f_T\in\mathcal{F}_{\alpha}} R^{\pat}(T,P) \geq c_2\cdot  (d^{2}P)^{\frac{2}{5}}T^{\frac{3}{5}},
\end{align*}
where $c_2>0$ is a constant indepedent of $d, T$ and $P$. 
\end{theorem}

\begin{proof}[Proof of Theorem \ref{thm:2lowerbound}]
The proof uses the same notation and follows the same steps as in the proof of Theorem~\ref{thm:lowerbound}, but with different choices for the parameters \(h\) and \(N_c\). Define the set of functions $f_{\bomega}:\mathbb{R}^{\dx}\to\mathbb{R}$ with $\bomega \in \{-1,1\}^{\dx}$ as they are defined in \eqref{lw:function}, and choose $h = \min(d^{-\frac{1}{2}}, \frac{P}{N_c\sqrt{d}})$ and  $N_c = \lfloor P^{\frac{4}{5}}T^{\frac{1}{5}}d^{-\frac{2}{5}}\rfloor$. Then we have that 
\begin{align}
&    \sum_{k=2}^{N_c}\norm{\bx^*(\bomega_{k-1}) - \bx^*(\bomega_{k})}
= \frac{h\iota}{\alpha}\sum_{k=2}^{N_c}\sqrt{\rho(\bomega_{k-1},\bomega_k)}\leq \frac{h\lbfuncconst}{\alpha}\sqrt{d}N_c \leq P\,,
\end{align}
for any $\iota\leq \frac{\alpha}{2}$. Following similar steps as in the proof of Theorem~\ref{thm:lowerbound} for large enough $T$ (when $h =\frac{P}{N_c\sqrt{d}}$) we get
\begin{align*}
\max_{\Gamma\in\Omega^{N_c}}\sum_{t=1}^{T}\Exp_{\Gamma,t}\left[f_{\bomega_{k_t}}(\bz_t) - \min_{\bx\in\com}f_{\bomega_{k_t}}(\bx)\right]
\geq 
\frac{h^{2}\lbfuncconst^2T\dx}{64\alpha}\geq c_2 (d^2P)^{\frac{2}{5}}T^{\frac{3}{5}}\,,
\end{align*}
where $c_2>0$ is independent of $d, T$ and $P$.
\end{proof}

%% file: appx_exo.tex
\section{Proofs for clipped Exploration by Optimization}\label{appx:ExO_proofs}
We restate and prove Theorem \ref{thm:ExO_adaptive_regret} which establishes an adaptive regret guarantee for cExO. In this section,  we use $\ip{\bp,f_t}=\Exp_{\bz\sim \bp}[f_t(\bz)]$ where  $\bp$ belongs to a probability simplex.
\ExOtheorem*
\begin{proof}[Proof of Theorem \ref{thm:ExO_adaptive_regret}] 
Consider an arbitrary interval $[a,b]$ of length $b-a+1\leq \sfB$, and notice that for any $\bq^\star\in \tilde{\Delta}$,
\begin{align}
\max_{\bu\in\Theta}\sum_{t=a}^b\Exp[f_t(\bz_t) - f_t(\bu)]
&=\underbrace{\sum_{t=a}^b \ip{\bp_t-\bq^\star,f_t}}_{\text{term I}} + \underbrace{\sum_{t=a}^b\Exp_{\bz\sim \bq^\star}[f_t(\bz)]-\min_{\bu\in\Theta}\sum_{t=a}^bf_t(\bu)     }_{\text{term II}}\,.
\end{align} 
In what follows, we choose a suitable $\bq^\star$ and bound term I and term II separately. 

Recall that the covering set $\mc{C}$ is assumed in Section \ref{sec:convex} to have a discretization error of $\varepsilon$, implying that there exists a $\bu_{\mc{C}}\in \mc{C}$ such that $\sum_{t=a}^bf_t(\bu_{\mc{C}})-\min_{\bu\in\Theta}\sum_{t=a}^bf_t(\bu)\leq \varepsilon \sfB$. 
Define $\bq^\star\in\tilde\Delta$ to be the   distribution with probability mass given by 
\begin{align}\label{eq:def_q_star}
    q^\star(\bz) = \begin{cases}
    1-\gamma(|\mc{C}|-1)& \text{ if }\bz = \bu_{\mc{C}}\\
     \gamma & \text{ otherwise}\,.
\end{cases}
\end{align}
This construction ensures that
\begin{equation}\label{eq:regret_q_star}
\text{term II}\leq (\varepsilon+2\gamma|\mc{C}|)\sfB\,.
\end{equation}
To bound term I, we first apply Lemma~\ref{lem:FTRL} to  
the sequence of Online Mirror Descent (OMD) updates $\bq_t \in\tilde\Delta$  and the sequence of loss estimates $\widehat{\bs}_t$ to obtain
\begin{align}
    \sum_{t=a}^b\ip{\bq_t-\bq^\star,\widehat{\bs}_t } \leq \frac{1}{\eta}\left(\text{KL}(\bq^\star|| \bq_a)+\sum_{t=a}^bS_{\bq_t}(\eta\widehat{ \bs}_t)\right)\,,\label{eq:OMD_interval_regret}
\end{align}
where by  the definition of $q^\star(\cdot)$ in \eqref{eq:def_q_star}, we have 
\begin{align}
    \text{KL}(\bq^\star||\bq_a) &= \sum_{\bz\in\mc{C}} q^\star(\bz)\log\left(\frac{q^\star(\bz)}{q_a(\bz)}\right)\nonumber\\
    &=(1-\gamma(|\mc{C}|-1))\log\left(\frac{1-\gamma(|\mc{C}|-1)}{q_a(\bu_{\mc{C}})}\right)+\sum_{\bz\in\mc{C}\setminus\{\bu_{\mc{C}}\}}\gamma\log\left(\frac{\gamma}{q_a(\bz)}\right)\nonumber\\
    &\leq \log(\gamma^{-1})\,.\label{eq:OMD_KL_bound}
\end{align}
Then applying \eqref{eq:OMD_interval_regret} and \eqref{eq:OMD_KL_bound}, we  have 
\begin{align}\label{eq:regret_pt_q_star}
   \text{term I}
   &=\sum_{t=a}^b\left[\ip{\bq_t-\bq^\star,\widehat{\bs}_t}+\ip{\bp_t-\bq^\star,f_t}+\ip{\bq^\star-\bq_t,\widehat{\bs}_t}\right]
    \nonumber\\
    &\leq \frac{\log(\gamma^{-1})}{\eta}+\sum_{t=a}^b\left[\ip{\bp_t-\bq^\star,f_t}+\ip{\bq^\star-\bq_t,\widehat{\bs}_t}+\frac{1}{\eta}S_{\bq_t}(\eta\widehat{\bs}_t)\right]
   \nonumber \\
    &\stackrel{\text{(i)}}{\le}\frac{\log(\gamma^{-1})}{\eta} + \sfB\left( \inf_{\substack{\bp\in \Delta(\mc{C}),\\E\in \mc{E} }}\Lambda_\eta(\bq_t, \bp, E) + \eta d\right) \nonumber\\
    & \stackrel{\text{(ii)}}{\leq} \frac{\log(\gamma^{-1})}{\eta}+\sfB\left(\eta \kappa d^4\log(dT)+ \eta d\right)\,,
\end{align}
where (i) follows from the update rule  \eqref{eq:def_Lambda_eta} and the precision level assumed for solving the minimization problem \eqref{eq:def_Lambda_eta} (see line~\ref{line:ExO_pt} of Algorithm~\ref{alg:convex}), and (ii) uses \cite[Theorems 8.19 and 8.21]{lattimore24introbco} which establish that there exists a universal constant $\kappa$ such that  
\begin{align}
    \sup_{\bq\in\tilde\Delta}  \inf_{\substack{\bp\in \Delta(\mc{C}),\\E\in \mc{E} }} \frac{1}{\eta}\Lambda_\eta(\bq, \bp, E) \leq  \kappa d^4\log(dT)\,.\nonumber
\end{align} 
Finally, combining \eqref{eq:regret_q_star}  and \eqref{eq:regret_pt_q_star} we obtain
\begin{align}
\max_{\bu\in\Theta}\sum_{t=a}^b\Exp[f_t(\bz_t) - f_t(\bu)]
&=\sum_{t=a}^b \ip{\bp_t-\bq^\star,f_t} + \sum_{t=a}^b\Exp_{\bz\sim \bq^\star}[f_t(\bz)]-\min_{\bu\in\Theta}\sum_{t=a}^bf_t(\bu)\nonumber\\
& \le (\varepsilon+2\gamma|\mc{C}|)\sfB + 
\frac{\log(\gamma^{-1})}{\eta}+\sfB\left(\eta \kappa d^4\log(dT)+ \eta d\right)\nonumber\\
& \lesssim \frac{\sfB}{T}+ \sqrt{\sfB   d^4 \log(dT) \log( T|\mc{C}|)}\lesssim d^\frac{5}{2}\sqrt{\sfB}\,,\nonumber
\end{align}
where (i) applies $\varepsilon=\frac{1}{T}$, $\gamma = \frac{1}{T|\mc{C}|}$ and  $\eta=\sqrt{\log(\gamma^{-1})/(d^4\log(dT)\sfB)}$, and (ii) is by selecting the covering set $\mc{C}$ such  that   $\log|\mc{C}|\leq d\log(1+16dT^2)$ (existence given by \cite[Definition 8.12 and Exercise 8.13]{lattimore24introbco}). 
\end{proof}

The proof of Theorem \ref{thm:ExO_adaptive_regret} above relied on Lemma \ref{lem:FTRL}, which we present and prove below.
\begin{lemma}
\label{lem:FTRL}
Consider  Online Mirror Descent (OMD) with KL divergence regularization and fixed learning rate $\eta>0$ applied to  a sequence of  loss estimates 
  $\bs_t \in \bbR^{n}$ for ${t\in \bbN^+}$.   When run  over a convex and complete domain $\tilde\Delta\subseteq \Delta^{n-1}$, the algorithm produces a sequence of updates $\bq_t\in \tilde{\Delta}$ for $t\in \bbN^+$. For any comparator in $\bq^\star\in\tilde\Delta$ and time interval $\{t\in \bbN^+: a\le t\le b\}$, it holds that
    \begin{align*}
        \sum_{t=a}^{b}\ip{\bq_t-\bq^\star,\bs_t} &\leq \frac{1}{\eta}\left(\text{KL}(\bq^\star||\bq_a) +\sum_{t=a}^{b}S_{\bq_t}(\eta\bs_t)\right)\,,
    \end{align*}
    where $S_{\bq}(\eta\bs)=\max_{\bq'\in \Delta(\mc{C})}\ip{\bq-\bq', \eta\bs} - \text{KL}(\bq'||\bq)$.
\end{lemma}
\begin{proof}[Proof of Lemma \ref{lem:FTRL}]
The proof is standard and included for completeness.
Let $F$ denote the neg-entropy $F(\bq)=\sum_{i=1}^{n}q_i\log(q_i)$ for $\bq\in \Delta^{n-1}$, and note that 
\begin{align}\label{eq:KL_and_entropy}
    \text{KL}(\bp||\bq)=F(\bp)-\ip{\bp-\bq,\nabla F(\bq)}-F(\bq)\quad \forall \; \bp, \bq\in \Delta^{n-1}\,.
\end{align} 
Consider the update rule of the OMD defined in the lemma:
\begin{align*}
\bq_{t+1}=\argmin_{\bq\in\tilde\Delta} \ip{\bq,\eta\bs_t}+\text{KL}(\bq||\bq_t)=\argmin_{\bq\in\tilde\Delta} \ip{\bq,\eta\bs_t}+F(\bq)-\bq \nabla  F(\bq_{t})\,,
\end{align*}
which implies by the first order optimality condition \cite[Proposition 26.14]{lattimore2020bandit} that, for any $\bq^\star\in \tilde\Delta$ and  time  $t$,
\begin{align}
\ip{\bq^\star-\bq_{t+1},\eta \bs_t +\nabla F(\bq_{t+1}) - \nabla F(\bq_t)} \geq 0\,.\label{eq:first_order_optimality}
\end{align}
Rearranging \eqref{eq:first_order_optimality} and applying \eqref{eq:KL_and_entropy}  we obtain 
\begin{align}
    \ip{\bq_{t+1}-\bq^\star,\bs_t} &\leq \frac{1}{\eta}\ip{\bq^\star-\bq_{t+1},\nabla F(\bq_{t+1})-\nabla F(\bq_{t})}\nonumber
    \\&= \frac{1}{\eta}\left(\text{KL}(\bq^\star||\bq_{t})-\text{KL}(\bq^\star||\bq_{t+1})-\text{KL}(\bq_{t+1}||\bq_{t})\right) \nonumber
    \\
    &\le -\ip{\bq_t-\bq_{t+1},\bs_t} +\frac{1}{\eta}S_{\bq_t}(\eta\bs_t)+ \frac{1}{\eta}\left(\text{KL}(\bq^\star||\bq_{t})-\text{KL}(\bq^\star||\bq_{t+1})\right)\,.
    \label{eq:first_order_optimality1}
\end{align}
Rearranging \eqref{eq:first_order_optimality1} and summing over $t\in [a,b]$ yields
\begin{align*}
\sum_{t=a}^b \ip{\bq_t-\bq^\star,\bs_t}&=\sum_{t=a}^b\left(\ip{\bq_{t+1}-\bq^\star,\bs_t} +\ip{\bq_t-\bq_{t+1},\bs_t}\right)\\
    &\leq \frac{1}{\eta}\sum_{t=a}^b\left( \text{KL}(\bq^\star||\bq_{t})-\text{KL}(\bq^\star||\bq_{t+1}) +S_{\bq_t}(\eta\bs_t)\right)
    \\&=\frac{1}{\eta}\left(  \text{KL}(\bq^\star||\bq_a)-\text{KL}(\bq^\star||\bq_{b+1}) +\sum_{t=a}^b S_{\bq_t}(\eta\bs_t)\right)\,,
\end{align*}
which combined with non-negativity of the KL divergence completes the proof.
\end{proof}

Finally, we apply Theorem \ref{thm:ExO_adaptive_regret} to prove the  bounds on $R^\swi(T,S), R^\dyn(T, \Delta, S)$ and $R^\pat(T, P)$  in Corollary \ref{cor:convex}, as well as the parameter-free guarantees in Corollary \ref{cor:BoB_ExO}. 
\ExOcorollaryKnown*
\begin{proof}[Proof of Corollary \ref{cor:convex}]
  We prove these results by applying the adaptive regret guarantee from Theorem~\ref{thm:ExO_adaptive_regret} and the conversions results from Proposition \ref{prop:conversions},  similarly to the proof of Corollary \ref{cor:TEWA_R_swi_known}.
\end{proof}
\begin{corollary}[cExO with BoB]\label{cor:BoB_ExO}
Let $T\in \bbN^+$. By partitioning the time horizon $[T]$ into epochs of length $L=d^{\frac{5}{2}}\sqrt{T}$, and  employing Bandit-over-Bandit  to select cExO's parameter $\sfB$ for each  epoch from the set $\mc{B}=\{2^i: i=0, 1, \dots, \lfloor\log_2T\rfloor\}$, this algorithm achieves all  regret bounds in Corollary~\ref{cor:convex} with an additional term of $d^{\frac{5}{4}}T^{\frac{3}{4}} $  (up to polylogarithmic factors). 
\end{corollary}
\begin{proof}[Proof of Corollary \ref{cor:BoB_ExO}]
    The proof is similar to that of Corollary \ref{cor:BoB} and is therefore omitted.
\end{proof}

%% file: bibliography.bib
@book{lattimore2020bandit,
  title={Bandit algorithms},
  author={Lattimore, Tor and Szepesv{\'a}ri, Csaba},
  year=2020,
  publisher={Cambridge University Press}
}

@article{hall2015online,
  title={Online convex optimization in dynamic environments},
  author={Hall, Eric C and Willett, Rebecca M},
  journal={IEEE Journal of Selected Topics in Signal Processing},
  volume={9},
  number={4},
  pages={647--662},
  year={2015},
  publisher={IEEE}
}

@inproceedings{luo2018efficient,
  title={Efficient contextual bandits in non-stationary worlds},
  author={Luo, Haipeng and Wei, Chen-Yu and Agarwal, Alekh and Langford, John},
  booktitle={Conference On {L}earning {T}heory},
  pages={1739--1776},
  year={2018},
  organization={PMLR}
}

@article{besbes2014stochastic,
  title={Stochastic multi-armed-bandit problem with non-stationary rewards},
  author={Besbes, Omar and Gur, Yonatan and Zeevi, Assaf},
  journal={Advances in Neural Information Processing Systems},
  volume={27},
  year={2014}
}

@inproceedings{zinkevich2003online,
  title={Online convex programming and generalized infinitesimal gradient ascent},
  author={Zinkevich, Martin},
  booktitle={International Conference on Machine Learning},
  pages={928--936},
  year={2003}
}

@article{akhavan2024contextual,
  title={Contextual Continuum Bandits: Static Versus Dynamic Regret},
  author={Akhavan, Arya and Lounici, Karim and Pontil, Massimiliano and Tsybakov, Alexandre B},
  journal={arXiv preprint arXiv:2406.05714},
  year={2024}
}

@article{auer2002nonstochastic,
  title={The nonstochastic multiarmed bandit problem},
  author={Auer, Peter and Cesa-Bianchi, Nicolo and Freund, Yoav and Schapire, Robert E},
  journal={SIAM journal on computing},
  volume={32},
  number={1},
  pages={48--77},
  year={2002},
  publisher={SIAM}
}

@inproceedings{lattimore2021mirror,
  title={Mirror descent and the information ratio},
  author={Lattimore, Tor and Gyorgy, Andras},
  booktitle={Conference on {L}earning {T}heory},
  pages={2965--2992},
  year={2021},
  organization={PMLR}
}

@BOOK{Tsybakov09,
  title = {Introduction to nonparametric estimation},
  publisher = {Springer},
  year = {2009},
  author = {Tsybakov, Alexandre B. },
  series = {Springer Series in Statistics},
  address = {New York}
}

@article{besbes2015non,
  title={Non-stationary stochastic optimization},
  author={Besbes, Omar and Gur, Yonatan and Zeevi, Assaf},
  journal={Operations Research},
  volume={63},
  number={5},
  pages={1227--1244},
  year={2015},
  publisher={INFORMS}
}

@article{chen2018bandit,
  title={Bandit convex optimization for scalable and dynamic {I}o{T} management},
  author={Chen, Tianyi and Giannakis, Georgios B},
  journal={IEEE Internet of Things Journal},
  volume={6},
  number={1},
  pages={1276--1286},
  year={2018},
  publisher={IEEE}
}

@inproceedings{zhang2018dynamic,
  title={Dynamic regret of strongly adaptive methods},
  author={Zhang, Lijun and Yang, Tianbao and Zhou, Zhi-Hua and others},
  booktitle={International conference on machine learning},
  pages={5882--5891},
  year={2018},
  organization={PMLR}
}

@inproceedings{zhang2018adaptive,
author = {Zhang, Lijun and Lu, Shiyin and Zhou, Zhi-Hua},
title = {Adaptive online learning in dynamic environments},
year = {2018},
publisher = {Curran Associates Inc.},
booktitle = {Proceedings of the 32nd International Conference on Neural Information Processing Systems},
pages = {1330–1340},
numpages = {11},
}

@inproceedings{chen2019new,
  title={A new algorithm for non-stationary contextual bandits: Efficient, optimal and parameter-free},
  author={Chen, Yifang and Lee, Chung-Wei and Luo, Haipeng and Wei, Chen-Yu},
  booktitle={Conference on {L}earning {T}heory},
  pages={696--726},
  year={2019},
  organization={PMLR}
}

@article{zhao2021bandit,
  title={Bandit convex optimization in non-stationary environments},
  author={Zhao, Peng and Wang, Guanghui and Zhang, Lijun and Zhou, Zhi-Hua},
  journal={Journal of Machine Learning Research},
  volume={22},
  number={125},
  pages={1--45},
  year={2021}
}

@inproceedings{hazan2014bandit,
  title={Bandit convex optimization: Towards tight bounds},
  author={Hazan, Elad and Levy, Kfir},
  booktitle = {Advances in Neural Information Processing Systems},
 pages = {},
 publisher = {Curran Associates, Inc.},
 volume = {27},
 year = {2014}
}

@inproceedings{kleinberg2004nearly,
author = {Kleinberg, Robert},
title = {Nearly tight bounds for the continuum-armed bandit problem},
year = {2004},
publisher = {MIT Press},
booktitle = {International Conference on Neural Information Processing Systems},
pages = {697–704},
numpages = {8},
}

@inproceedings{flaxman2004online,
  author       = {Abraham Flaxman and
                  Adam Tauman Kalai and
                  H. Brendan McMahan},
  title        = {Online convex optimization in the bandit setting: gradient descent
                  without a gradient},
  booktitle    = {Proceedings of the Sixteenth Annual {ACM-SIAM} Symposium on Discrete
                  Algorithms},
  pages        = {385--394},
  publisher    = {{SIAM}},
  year         = {2005}
}

@inproceedings{lattimore2023second,
  title={A second-order method for stochastic bandit convex optimisation},
  author={Lattimore, Tor and Gy{\"o}rgy, Andr{\'a}s},
  booktitle={Conference on {L}earning {T}heory},
  pages={2067--2094},
  year={2023},
  organization={PMLR}
}

@inproceedings{ito2020optimal,
  title={An optimal algorithm for bandit convex optimization with strongly-convex and smooth loss},
  author={Ito, Shinji},
  booktitle={International Conference on Artificial Intelligence and Statistics},
  pages={2229--2239},
  year={2020},
  organization={PMLR}
}

@article{auer2002finite,
  title={Finite-time analysis of the multiarmed bandit problem},
  author={Auer, Peter and Cesa-Bianchi, Nicolo and Fischer, Paul},
  journal={Machine learning},
  volume={47},
  pages={235--256},
  year={2002},
  publisher={Springer}
}

@inproceedings{shamir2013complexity,
  title={On the complexity of bandit and derivative-free stochastic convex optimization},
  author={Shamir, Ohad},
  booktitle={{Conference on Learning Theory}},
  pages={3--24},
  year={2013},
  organization={PMLR}
}

@article{akhavan2020exploiting,
  title={Exploiting higher order smoothness in derivative-free optimization and continuous bandits},
  author={Akhavan, Arya and Pontil, Massimiliano and Tsybakov, Alexandre},
  journal={Advances in Neural Information Processing Systems},
  volume={33},
  pages={9017--9027},
  year={2020}
}

@article{orabona2019modern,
  title={A modern introduction to online learning},
  author={Orabona, Francesco},
  journal={arXiv:1912.13213},
  year={2019}
}

@article{lattimore2020improved,
  title={Improved regret for zeroth-order adversarial bandit convex optimisation},
  author={Lattimore, Tor},
  journal={Mathematical Statistics and Learning},
  volume={2},
  number={3},
  pages={311--334},
  year={2020}
}

@InProceedings{auer2019adswitch,
  title = 	 {Adaptively Tracking the Best Bandit Arm with an Unknown Number of Distribution Changes},
  author =       {Auer, Peter and Gajane, Pratik and Ortner, Ronald},
  booktitle = 	 {Conference on Learning Theory},
  pages = 	 {138--158},
  year = 	 {2019},
  volume = 	 {99},
  publisher =    {PMLR},
}

@inproceedings{WeiL21,
title={Non-stationary reinforcement learning without prior knowledge: An optimal black-box approach},
  author={Wei, Chen-Yu and Luo, Haipeng},
  booktitle={{Conference on learning theory}},
  pages={4300--4354},
  year={2021},
  organization={PMLR}
}

@inproceedings{HazanS09,
  author       = {Elad Hazan and
                  C. Seshadhri},
  title        = {Efficient learning algorithms for changing environments},
  booktitle    = { International Conference on Machine  Learning},
  series       = {{ACM} International Conference Proceeding Series},
  volume       = {382},
  pages        = {393--400},
  publisher    = {{ACM}},
  year         = {2009},
}

@InProceedings{cheung19bob,
  title = 	 {Learning to Optimize under Non-Stationarity},
  author =       {Cheung, Wang Chi and Simchi-Levi, David and Zhu, Ruihao},
  booktitle = 	 {International Conference on Artificial Intelligence and Statistics},
  pages = 	 {1079--1087},
  year = 	 {2019},
  volume = 	 {89},
  series = 	 {Proceedings of Machine Learning Research},
  publisher =    {PMLR}
}

@inproceedings{SukK22,
  author       = {Joe Suk and
                  Samory Kpotufe},
  title        = {Tracking Most Significant Arm Switches in Bandits},
  booktitle    = {Conference on Learning Theory},
  series       = {Proceedings of Machine Learning Research},
  volume       = {178},
  pages        = {2160--2182},
  year         = {2022}
}

@InProceedings{cutkosky2020parameter,
  title = 	 {Parameter-free, Dynamic, and Strongly-Adaptive Online Learning},
  author =       {Cutkosky, Ashok},
  booktitle = 	 {International Conference on Machine Learning},
  pages = 	 {2250--2259},
  year = 	 {2020},
  volume = 	 {119},
  publisher =    {PMLR}
}

@article{jun2017online,
      title={Online Learning for Changing Environments using Coin Betting}, 
      author={Kwang-Sung Jun and Francesco Orabona and Stephen Wright and Rebecca Willett},
      year={2017},
journal   = {arXiv preprint arXiv:1711.02545}
}

@article{freud1997predictors_specialise,
author = {Freund, Yoav and Schapire, Robert and Singer, Yoram and Warmuth, Manfred},
year = {1997},
month = {01},
pages = {},
title = {Using and combining predictors that specialize},
journal = {Conference Proceedings of the Annual ACM Symposium on Theory of Computing}
}

@inproceedings{wang2018minimizing,
author = {Wang, Guanghui and Zhao, Dakuan and Zhang, Lijun},
title = {Minimizing adaptive regret with one gradient per iteration},
year = {2018},
publisher = {AAAI Press},
booktitle = {International Joint Conference on Artificial Intelligence},
pages = {2762–2768},
numpages = {7},
series = {IJCAI'18}
}

@inproceedings{zhang2021dualadaptivity,
author = {Zhang, Lijun and Wang, Guanghui and Tu, Wei-Wei and Jiang, Wei and Zhou, Zhi-Hua},
title = {Dual adaptivity: a universal algorithm for minimizing the adaptive regret of convex functions},
year = {2021},
publisher = {Curran Associates Inc.},
booktitle = { International Conference on Neural Information Processing Systems},
articleno = {1912},
numpages = {13},
}

@article{zhao2022efficient,
  title={Efficient methods for non-stationary online learning},
  author={Zhao, Peng and Xie, Yan-Feng and Zhang, Lijun and Zhou, Zhi-Hua},
  journal={Advances in Neural Information Processing Systems},
  volume={35},
  pages={11573--11585},
  year={2022}
}

@article{vanerven2016metagrad,
  author  = {Tim van Erven and Wouter M. Koolen and Dirk van der Hoeven},
  title   = {MetaGrad: Adaptation using Multiple Learning Rates in Online Learning},
  journal = {Journal of Machine Learning Research},
  year    = {2021},
  volume  = {22},
  number  = {161},
  pages   = {1--61}
}

@inproceedings{koolen2015squint,
  title={Second-order quantile methods for experts and combinatorial games},
  author={Koolen, Wouter M and Van Erven, Tim},
  booktitle={Conference on Learning Theory},
  pages={1155--1175},
  year={2015},
  organization={PMLR}
}

@InProceedings{wang2020adaptivity,
  title = 	 {Adaptivity and Optimality: A Universal Algorithm for  Online Convex Optimization},
  author =       {Wang, Guanghui and Lu, Shiyin and Zhang, Lijun},
  booktitle = 	 {Proceedings of The 35th Uncertainty in Artificial Intelligence Conference},
  pages = 	 {659--668},
  year = 	 {2020},
  volume = 	 {115},
  series = 	 {Proceedings of Machine Learning Research},
  publisher =    {PMLR}
}

@article{herbster1998tracking,
	title = {Tracking the {Best} {Expert}},
	volume = {32},
	number = {2},
	journal = {Machine Learning},
	author = {Herbster, Mark and Warmuth, Manfred K.},
	month = aug,
	year = {1998},
	pages = {151--178},
}

@inproceedings{hazan2007adaptive_ogd,
  title={Adaptive online gradient descent},
author = {Bartlett, Peter L. and Hazan, Elad and Rakhlin, Alexander},
title = {Adaptive online gradient descent},
year = {2007},
publisher = {Curran Associates Inc.},
booktitle = {Proceedings of the 21st International Conference on Neural Information Processing Systems},
pages = {65–72},
numpages = {8}
}

@article{bousquet2002tracking,
  title={Tracking a small set of experts by mixing past posteriors},
  author={Bousquet, Olivier and Warmuth, Manfred K},
  journal={Journal of Machine Learning Research},
  volume={3},
  number={Nov},
  pages={363--396},
  year={2002}
}

@inproceedings{daniely2015strongly,
  title={Strongly adaptive online learning},
  author={Daniely, Amit and Gonen, Alon and Shalev-Shwartz, Shai},
  booktitle={International Conference on Machine Learning},
  pages={1405--1411},
  year={2015},
  organization={PMLR}
}

@article{lu2022adaptive,
  title={Adaptive gradient methods with local guarantees},
  author={Lu, Zhou and Xia, Wenhan and Arora, Sanjeev and Hazan, Elad},
  journal={arXiv preprint arXiv:2203.01400},
  year={2022}
}

@inproceedings{garivier2008upper,
 author    = {Aur{\'e}lien Garivier and Eric Moulines},
  title     = {On Upper-Confidence Bound Policies for Switching Bandit Problems},
  booktitle = {Proceedings of the 22nd International Conference on Algorithmic Learning Theory},
  year      = {2011}
}

@article{Hazan16intoduction,
  author       = {Elad Hazan},
  title        = {Introduction to Online Convex Optimization},
  journal      = {Foundations and Trends{\textregistered} in Optimization},
  volume       = {2},
  number       = {3-4},
  pages        = {157--325},
  year         = {2016},
publisher={Now Publishers, Inc.}
}

@article{lattimore24introbco,
  author       = {Tor Lattimore},
  title        = {Bandit Convex Optimisation},
  journal={arXiv:2402.06535},
  year={2024}
}

@article{KL_UCB,
	title={{Kullback}-{Leibler} Upper Confidence Bounds for Optimal Sequential Allocation},
	author={Capp{\'e}, Olivier and Garivier, Aur{\'e}lien and Maillard, Odalric-Ambrym and Munos, R{\'e}mi and Stoltz, Gilles},
	journal={Annals of Statistics},
	volume={41},
	number={3},
	pages={1516--1541},
	year={2013}
}

@article{TS_1933,
	title={On the likelihood that one unknown probability exceeds another in view of the evidence of two samples},
	author={Thompson, William R},
	journal={Biometrika},
	volume={25},
	number={3/4},
	pages={285--294},
	year={1933},
	publisher={JSTOR}
	}

@inproceedings{TS12AG,
	author    = {Shipra Agrawal and
	Navin Goyal},
	title     = {Analysis of {Thompson} Sampling for the Multi-armed Bandit Problem},
	booktitle = {Conference on Learning Theory},
	year      = {2012}
}

@inproceedings{TS12kaufmann,
	title={Thompson sampling: An asymptotically optimal finite-time analysis},
  author={Kaufmann, Emilie and Korda, Nathaniel and Munos, R{\'e}mi},
  booktitle={International conference on algorithmic learning theory},
  pages={199--213},
  year={2012},
  organization={Springer}
}

@inproceedings{LiuLS18,
  author       = {Fang Liu and
                  Joohyun Lee and
                  Ness B. Shroff},
  title        = {A Change-Detection Based Framework for Piecewise-Stationary Multi-Armed
                  Bandit Problem},
  booktitle={Proceedings of the AAAI Conference on Artificial Intelligence},
  volume={32},
  number={1},
  year={2018}
}

@inproceedings{Cao19,
  author       = {Yang Cao and
                  Zheng Wen and
                  Branislav Kveton and
                  Yao Xie},
  title={Nearly optimal adaptive procedure with change detection for piecewise-stationary bandit},
  author={Cao, Yang and Wen, Zheng and Kveton, Branislav and Xie, Yao},
  booktitle={International Conference on Artificial Intelligence and Statistics},
  pages={418--427},
  year={2019},
  organization={PMLR}
}

@article{BessonKMS22,
  author       = {Lilian Besson and
                  Emilie Kaufmann and
                  Odalric{-}Ambrym Maillard and
                  Julien Seznec},
  title        = {Efficient Change-Point Detection for Tackling Piecewise-Stationary
                  Bandits},
  journal      = {Journal of Machine Learning Research},
    year    = {2022},
  volume  = {23},
  number  = {77},
  pages   = {1--40},
}

@article{RussacVC19,
 title={Weighted linear bandits for non-stationary environments},
  author={Russac, Yoan and Vernade, Claire and Capp{\'e}, Olivier},
  journal={Advances in Neural Information Processing Systems},
  volume={32},
  year={2019}
}

@inproceedings{baudry21ns,
	title={On limited-memory subsampling strategies for bandits},
  author={Baudry, Dorian and Russac, Yoan and Capp{\'e}, Olivier},
  booktitle={International Conference on Machine Learning},
  pages={727--737},
  year={2021},
  organization={PMLR}
}

@article{TrovoRG20,
  author       = {Francesco Trov{\`{o}} and
                  Marcello Restelli and
                  Nicola Gatti},
  title        = {Sliding-Window Thompson Sampling for Non-Stationary Settings},
  journal      = {Journal of  Artificial Intelligence Research},
  volume       = {68},
  pages        = {311--364},
  year         = {2020}
}

@inproceedings{Jadbabaie15,
  author    = {A. Jadbabaie and A. Rakhlin and S. Shahrampour and K. Sridharan},
  title     = {Online optimization: Competing with dynamic comparators},
  booktitle={Artificial Intelligence and Statistics},
  pages={398--406},
  year={2015},
  organization={PMLR}
}

@inproceedings{Mokhtari16,
  title={Online optimization in dynamic environments: Improved regret rates for strongly convex problems},
  author={Mokhtari, Aryan and Shahrampour, Shahin and Jadbabaie, Ali and Ribeiro, Alejandro},
  booktitle={Conference on Decision and Control},
  pages={7195--7201},
  year={2016},
  organization={IEEE}
}

@article{agarwal2011stochastic,
  title={Stochastic convex optimization with bandit feedback},
  author={Agarwal, Alekh and Foster, Dean P and Hsu, Daniel J and Kakade, Sham M and Rakhlin, Alexander},
  journal={Advances in Neural Information Processing Systems},
  volume={24},
  year={2011}
}

@inproceedings{saha2011improved,
  title={Improved regret guarantees for online smooth convex optimization with bandit feedback},
  author={Saha, Ankan and Tewari, Ambuj},
  booktitle={  International conference on artificial intelligence and statistics},
  pages={636--642},
  year={2011},
  organization={JMLR Workshop and Conference Proceedings}
}

@article{hazan2007logarithmic,
title={Logarithmic regret algorithms for online convex optimization}, volume={69}, 
 number={2}, journal={Machine Learning}, author={Hazan, Elad and Agarwal, Amit and Kale, Satyen}, year={2007}, month=dec, pages={169–192}
}

@article{bubeck2021kernel,
  title={Kernel-based methods for bandit convex optimization},
  author={Bubeck, S{\'e}bastien and Eldan, Ronen and Lee, Yin Tat},
  journal={Journal of the ACM },
  volume={68},
  number={4},
  pages={1--35},
  year={2021},
  publisher={ACM New York, NY}
}

@InProceedings{fokkema2024online,
  title={Online newton method for bandit convex optimisation},
  author={Fokkema, Hidde and van der Hoeven, Dirk and Lattimore, Tor and Mayo, Jack J},
  booktitle = 	 {Conference on Learning Theory},
  pages = 	 {1713--1714},
  year = 	 {2024},
  volume = 	 {247},
  series = 	 {Proceedings of Machine Learning Research},
  publisher =    {PMLR},
}

@article{wang2025adaptivity,
  title={On adaptivity in nonstationary stochastic optimization with bandit feedback},
  author={Wang, Yining},
  journal={Operations Research},
  volume={73},
  number={2},
  pages={819--828},
  year={2025},
  publisher={INFORMS}
}

@article{adamskiy2016closer,
  author  = {Dmitry Adamskiy and Wouter M. Koolen and Alexey Chernov and Vladimir Vovk},
  title   = {A Closer Look at Adaptive Regret},
  journal = {Journal of Machine Learning Research},
  year    = {2016},
  volume  = {17},
  number  = {23},
  pages   = {1--21}
}

@article{Vovk1998game,
title = {A Game of Prediction with Expert Advice},
journal = {Journal of Computer and System Sciences},
volume = {56},
number = {2},
pages = {153-173},
year = {1998},
author = {V Vovk},
}

@article{cesaBianchi1997howto,
author = {Cesa-Bianchi, Nicol\`{o} and Freund, Yoav and Haussler, David and Helmbold, David P. and Schapire, Robert E. and Warmuth, Manfred K.},
title = {How to use expert advice},
year = {1997},
issue_date = {May 1997},
publisher = {Association for Computing Machinery},
address = {New York, NY, USA},
volume = {44},
number = {3},
journal = {Journal of the ACM},
month = may,
pages = {427–485},
numpages = {59},
}

@article{littlestone1994weighted,
title = {The Weighted Majority Algorithm},
journal = {Information and Computation},
volume = {108},
number = {2},
pages = {212-261},
year = {1994},
author = {N. Littlestone and M.K. Warmuth},
}

@inproceedings{zhang2020minimizing,
  title={Minimizing dynamic regret and adaptive regret simultaneously},
  author={Zhang, Lijun and Lu, Shiyin and Yang, Tianbao},
  booktitle={International Conference on Artificial Intelligence and Statistics},
  pages={309--319},
  year={2020},
  organization={PMLR}
}

@inproceedings{baby2022optimal,
  title={Optimal dynamic regret in proper online learning with strongly convex losses and beyond},
  author={Baby, Dheeraj and Wang, Yu-Xiang},
  booktitle={International Conference on Artificial Intelligence and Statistics},
  pages={1805--1845},
  year={2022},
  organization={PMLR}
}

@inproceedings{yang2024universal,
  title={Universal Online Convex Optimization with $1 $ Projection per Round},
  author={Yang, Wenhao and Wang, Yibo and Zhao, Peng and Zhang, Lijun},
  year={2024},
 booktitle = {Advances in Neural Information Processing Systems},
 pages = {31438--31472},
 publisher = {Curran Associates, Inc.},
 volume = {37}
}

@inproceedings{do2009proximal,
author = {Do, Chuong B. and Le, Quoc V. and Foo, Chuan-Sheng},
title = {Proximal regularization for online and batch learning},
year = {2009},
publisher = {Association for Computing Machinery},
booktitle = { International Conference on Machine Learning},
pages = {257–264},
numpages = {8}
}

@inproceedings{bubeck2015bandit,
  title={Bandit convex optimization: $\sqrt{T}$ regret in one dimension},
  author={Bubeck, S{\'e}bastien and Dekel, Ofer and Koren, Tomer and Peres, Yuval},
  booktitle={Conference on {L}earning {T}heory},
  pages={266--278},
  year={2015},
  organization={PMLR}
}

@article{bubeck2018exploratory,
  title={Exploratory distributions for convex functions},
  author={Bubeck, S{\'e}bastien and Eldan, Ronen},
  journal={Mathematical Statistics and Learning},
  volume={1},
  number={1},
  pages={73--100},
  year={2018}
}

@inproceedings{russo2014learning,
author = {Russo, Daniel and Van Roy, Benjamin},
 booktitle = {Advances in Neural Information Processing Systems},
 pages = {},
 publisher = {Curran Associates, Inc.},
 title = {Learning to Optimize via Information-Directed Sampling},
 volume = {27},
 year = {2014}
}

@phdthesis{stoltz2005thesis,
  TITLE = {{Incomplete information and internal regret in prediction of individual sequences}},
  AUTHOR = {Stoltz, Gilles},
  HAL_LOCAL_REFERENCE = {Ordre 7885},
  SCHOOL = {{Universit{\'e} Paris Sud - Paris XI}},
  YEAR = {2005},
  MONTH = May,
  TYPE = {Theses},
}

@inproceedings{suggala2024second,
  title={Second order methods for bandit optimization and control},
  author={Suggala, Arun and Sun, Y Jennifer and Netrapalli, Praneeth and Hazan, Elad},
  booktitle={The Thirty Seventh Annual Conference on Learning Theory},
  pages={4691--4763},
  year={2024},
  organization={PMLR}
}

@InProceedings{suggala2021efficient,
  title = 	 {Efficient Bandit Convex Optimization: Beyond Linear Losses},
  author =       {Suggala, Arun Sai and Ravikumar, Pradeep and Netrapalli, Praneeth},
  booktitle = 	 {Proceedings of Thirty Fourth Conference on Learning Theory},
  pages = 	 {4008--4067},
  year = 	 {2021},
  volume = 	 {134},
  publisher =    {PMLR},
}
